\documentclass[10pt,twocolumn,letterpaper]{article}

\usepackage[pagenumbers]{cvpr} 

\usepackage{graphicx}
\usepackage{amsmath}
\usepackage{amssymb}
\usepackage{booktabs}

\usepackage[accsupp]{axessibility}  

\usepackage[pagebackref,breaklinks,colorlinks]{hyperref}

\usepackage[capitalize]{cleveref}
\crefname{section}{Sec.}{Secs.}
\Crefname{section}{Section}{Sections}
\Crefname{table}{Table}{Tables}
\crefname{table}{Tab.}{Tabs.}

\def\cvprPaperID{10708} 
\def\confName{CVPR}
\def\confYear{2022}

\input{preamble.sty}

\usepackage{fancyhdr}
\fancypagestyle{footnotepage1}{\fancyhf{}\fancyfoot[L]{\footnotesize Accepted to CVPR 2022. 
\\ © 2022 IEEE. Personal use of this material is permitted. Permission from IEEE must be
obtained for all other uses, in any current or future media, including
reprinting/republishing this material for advertising or promotional purposes, creating new
collective works, for resale or redistribution to servers or lists, or reuse of any copyrighted
component of this work in other works.}}

\begin{document}

\title{ ZZ-Net: A Universal Rotation Equivariant Architecture for 2D Point Clouds}

\author{Georg Bökman$^\text{a}$, Fredrik Kahl$^\text{a}$, Axel Flinth$^\text{a,b}$ \\ {\small\texttt{bokman@chalmers.se, fredrik.kahl@chalmers.se, axel.flinth@umu.se}}
\vspace{.1cm}\\
$^\text{a}$Department of Electrical Engineering, Chalmers University of Technology \\
$^\text{b}$Department of Mathematics and Mathematical Statistics, Umeå University
}




\maketitle

\begin{abstract}
In this paper, we are concerned with rotation equivariance on 2D point cloud data. 
We describe a particular set of functions able to approximate any continuous rotation equivariant
and permutation invariant function.
Based on this result, we propose a novel neural network architecture for processing 2D point clouds 
and we prove its universality for approximating functions exhibiting these symmetries.

We also show how to extend the architecture to accept a set of 2D-2D correspondences as indata, while maintaining similar equivariance properties. Experiments are presented on the estimation of essential matrices in stereo vision.
\end{abstract}

\thispagestyle{footnotepage1}
\section{Introduction}

The need to interpret and process point clouds arises in many different application areas such as autonomous driving, augmented reality and robotics  \cite{izadi2011kinectfusion}. Basic problem examples in computer vision are classification, segmentation and object detection as well as correspondence problems in multiple view geometry \cite{rapp2020lidar} .
Considering as input object a point cloud or a pair of point clouds, 
it is a natural requirement that permuting the order of the
points doesn't change the object in question. Such a permutation should therefore not change the way the points are processed. 
This permutation symmetry needs to be considered when designing a neural network
for point cloud input, which is typically done by having \emph{equivariant} network layers.
Another possible symmetry is rotation of the point clouds about the origin. For an example of a single point cloud processing task that is rotation equivariant, see Figure~\ref{fig:equiv_stars}. We will also consider rotational symmetries for pairs of point clouds.
\begin{figure}
    \centering
    \includegraphics[width=0.9\columnwidth]{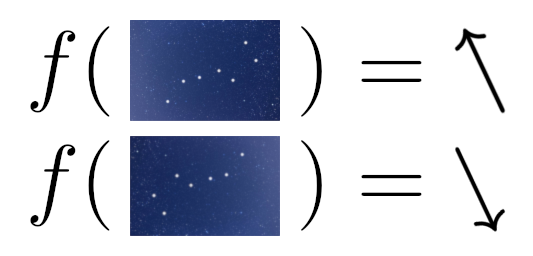}
    \caption{A simple example of rotation equivariance. The illustration shows the task of determining the direction to the North Star given other stars in the night sky. The input is the set of locations of the visible stars in some 2D coordinate frame. Rotation equivariance of the point cloud processor $f$ means here that the determined direction should rotate if the night sky (or the observer) rotates. (Picture of stars \cite{big_dipper_pic}.)}
    \label{fig:equiv_stars}
\end{figure}

\newcommand{\ifmathfrak}{}
\newcommand{\ifmathscr}{}

\textbf{Equivariance.}
\textcolor{black}{Let us introduce some notation and provide a formal definition of equivariance.} Given a group $\ifmathfrak G$ we consider sets which exhibit $\ifmathfrak G$-symmetries (in a sense to be made precise shortly)
and functions between such sets.
A $\ifmathfrak G$-set is a set $\ifmathscr X$ equipped with a $\ifmathfrak G$-action, i.e., a group homomorphism $\varphi$ from 
$\ifmathfrak G$ to the group of bijections from $\ifmathscr X$ to $\ifmathscr X$.
One should think of a $\ifmathfrak G$-action $\varphi$ as a way to relate elements of $\ifmathfrak G$ to symmetries of $\ifmathscr X$.
Typically, we will suppress the group homomorphism in the notation and write $\ifmathfrak gx$ or $\ifmathfrak g^*x$ for $\varphi(\ifmathfrak g)(x)$.
An example of a
$\ifmathfrak G$-set is $\C^m$ acted on by the permutation group $S_m$. Note that $S_m$ could act on $\C^m$
in different ways and we must specify the action to describe a $\ifmathfrak G$-set.
The canonical action is to permute the $m$ dimensions, but another obvious action is the trivial action
given by $\pi Z = Z$ for all $\pi\in S_m$ and all $Z\in\C^m$.
If we have two $\ifmathfrak G$-sets $\ifmathscr X$ and $\ifmathscr Y$, we say that a function $f: \ifmathscr X\to \ifmathscr Y$ is ($G$-)equivariant if it commutes with the $\ifmathfrak G$-actions: $f(\ifmathfrak gx)=\ifmathfrak gf(x)$. A special case is when the action on $\ifmathscr Y$ is trivial and then we call $f$ invariant: $f(\ifmathfrak gx)=f(x)$. \textcolor{black}{For more information on groups, symmetries and equivariance in general, we refer to e.g. \cite{kosmann-schwarzbachGroupsSymmetries2010}.}

In this paper, we focus on the  group $SO(2)\times S_m$. The $SO(2)$-action on a point cloud is
given by rotating all points about the origin and the $S_m$-action is given by permuting the
points.\footnote{Technical note: These two actions commute with each other and hence define an {$SO(2)\times S_m$-action}.}  More concretely, we are concerned with functions that are invariant to permutations, but equivariant to rotations. Let us call the set of such functions $\calR(m)$.

Additionally, we go further, and describe new results and neural network architectures for the case of clouds of pair of points, or correspondences. In this case, we deal with functions that are permutation invariant, rotation equivariant to one of the clouds, and rotation invariant with respect to the other. We call this set of functions $\calR_2(m)$.

An obvious limitation with our work is that we only deal with $SO(2)$-equivariance and not higher order rotations. Still, it is an important case with many different applications. For instance, in many scenarios, invariance with respect to rotation around one axis is the correct model. \textcolor{black}{Another example is essential matrix estimation \cite{hartleyMultipleViewGeometry2003}, which we will explore in Section \ref{sec:exp_ess}. Note that} the derivations are simplified and the computations can be made more efficient as the group of 2D rotations is commutative, which is not the case for $SO(d)$ with $d>2$.

The main contributions of this paper are as follows. First, we describe a dense set of equivariant functions on 2D point clouds (Theorem~\ref{th:denseset}). With that set as a basis,  we describe a neural network architecture for approximating the function space $\calR(m)$ and prove its universality (Theorem \ref{th:universality}). We then present how to extend that architecture to also cover $\calR_2(m)$ and discuss the extension's universality properties. We test our architecture on a (toy) rotation estimation problem and the estimation of essential matrices in stereo vision.

\subsection{Related work}
Equivariance for regular image grids has been studied in various settings, ranging from classical CNNs for translation invariance \cite{fukushima-1980,lecun-1989} to rotation and rigid transformation invariance \cite{worrall2017cvpr,weiler2018cvpr,weiler_cesa_2019}.
Equivariance on more general domains and under general groups has also been investigated in a recent line of research. In particular there has been a focus on describing linear equivariant functions, which can be alternated with non-linearities to obtain equivariant neural network architectures \cite{cohen2016group, cohenEquivariantConvolutionalNetworks2021, aronssonHomogeneousVectorBundles2021, finzi2021icml, lang2021ICLR}.
Recent surveys of the theory include \cite{bronsteinGeometricDeepLearning2021, gerkenGeometricDeepLearning2021, weilerCoordinateIndependentConvolutional2021}.

There exist a number of high-performing deep learning architectures for 3D point cloud processing,
mostly targeted for recognition, classification and segmentation, including methods that do not
take rotation equivariance into account
\cite{qi2017pointnet, Zhao_2021_ICCV} and methods that do consider the effects of rotations
\cite{fuchs2020se3transformers,  thomasTensorFieldNetworks2018, Melnyk_2021_ICCV,dengVectorNeuronsGeneral2021}.
The approach most similar in spirit to ours is \cite{Xu_2021_ICCV}, but while we let every 
point in the point cloud gather information from all others 
to obtain rotation invariant and permutation equivariant features, 
they use the sorted Gram matrix of local neighbourhoods to 
obtain local rotation and permutation invariant features. 
They do not prove the universality of their approach.

We focus on 2D rather than 3D. While the approaches for the 3D case could be modified
to apply to the 2D case as well, doing so would not take advantage of the fact that the 2D case is simpler.
Specifically, all rotations in 2D commute and this fact plays a crucial role in our proofs.

Our work is inspired by fundamental theoretical results in machine learning which aim to characterize equivariant point cloud networks. In the seminal work of \cite{zaheer2018deep}, all permutation equivariant functions were shown to belong to a particular family of functions from which equivariant network architectures can be constructed. In more recent work, the theory has been further developed and additional symmetries have been considered 
\cite{pmlr-v97-wagstaff19a,maron2018invariant,maron2019universality, keriven2019universal,yarotsky2021universal}. 
In \cite{dym-maron-2021}, the authors present a method for proving universality for rotation equivariant point cloud networks in 3D. 
Their proof technique is applicable to networks which allow latent features consisting of arbitrary high order tensors, such as for e.g\ Tensor Field Networks \cite{thomasTensorFieldNetworks2018}. In contrast, our networks only need to handle tensors of order two.

While finalizing this work, we were made aware of the concurrent papers \cite{villar2021scalars, yaoSimpleEquivariantMachine2021},
with an approach that is related to ours. In fact, Proposition~10 of \cite{villar2021scalars} is similar to our Theorem~\ref{th:denseset} but for the group $O(2)$ instead of $SO(2)$ (in fact, they deal with a $d$-dimensional underlying space and the group $O(d)$). In particular, we make a more thorough description and analysis of neural network architectures.

From an application point-of-view, we are interested in correspondence problems and more generally, robust fitting problems in multiple view geometry. State-of-the-art deep learning approaches in this context include early work such as CNe~\cite{moo-cvpr-2018} and OANet~\cite{zhang2019oanet} but also the more recent approaches ACNe~\cite{sun-cvpr-2020} where attentive context normalization is shown to improve permutation-equivariant learning and T-Net~\cite{Zhong_2021_ICCV} which also consists of a permutation equivariant network that is able to capture both global and channel-wise contextual information. However, these methods only incorporate permutation equivariance, which make them dependent on the coordinate frame of the points.
We give experimental comparisons to some of these approaches.

\textbf{Notation.}
Throughout the entire paper, we will identify $\R^2$  with $\C$. The group $SO(2)$ of rotations is then naturally identified with the unit circle $\mathbb{S}\subset \C$. To keep things simple, we understand point clouds as vectors $Z~=~(z_0, \dots, z_{m-1})\in \C^m$, where $m$ is the number of points. Note that the action of $SO(2)$ on $\C^m$ can be simply written $\theta Z$, where $\theta\in\mathbb{S}$ and that this can be equivalently read as complex multiplication or an action of the rotation group. 
We  write $[m]$ for the set of indices from $0$ to $m-1$.
The group of permutations is denoted $S_m$, and for $\pi \in S_m$, we let $\pi^*Z$ denote the permuted version of $Z$, i.e., $[\pi^*Z]_i = Z_{\pi^{-1}(i)}$. As in \cite{maron2018invariant}, we extend the latter to tensors: for $T \in (\C^{m})^{\otimes 2}$, $[\pi^*T]_{ij}=T_{\pi^{-1}(i)\pi^{-1}(j)}$. Let us further denote the subgroup  of permutations which fix the $0$-element, i.e., $\set{\pi \in S_m \, \vert \,\pi(0)=0}$  with $\Stab(0)$, which is called the \emph{stabilizer} of 0.
Finally, we let $\tau_i\in S_m$ be the transposition of $i$ and $0$.

\section{Approximating functions in \texorpdfstring{$\calR(m)$}{R(m)}}
In this section we describe the theory underlying our permutation invariant, rotation equivariant neural network architecture. We denote the set of continuous rotation equivariant and permutation invariant functions, i.e., functions $f:\C^m \to \C$ with $f(\theta \pi^*Z) =\theta f(Z)$ for all $\pi \in S_m$ and $\theta \in \S$ with $\calR(m)$.\footnote{$\calR$ for \emph{rotation}.} Throughout the paper, $m$ is fixed.

\subsection{A dense set of functions in \texorpdfstring{$\calR(m)$}{R(m)}}

To get an idea of how to design a network for approximating functions on $\calR(m)$, let us look at the DeepSet\cite{zaheer2018deep}, or PointNet\cite{qi2017pointnet}, architectures. In a nutshell, the reason that they are universal for approximating permutation invariant functions is that all such functions can be written as
   $\chi(\sum_{i \in [m]} \varrho(z_i) )$
 for some $K\in \N$ and functions $\varrho:\C \to \R^K$ and $\chi:\R^K \to \C$. A natural Ansatz for approximating functions in $\calR(m)$ is therefore to use a network of the same structure, but letting $\varrho$ and $\chi$ be rotation equivariant. Unfortunately, this simple idea will provably not succeed.
 \begin{prop} \label{prop:nogo1}
    For any $m \geq 5$, there are functions $f \in \calR(m)$ that cannot be uniformly approximated only using functions as
    $\chi(\sum_{i \in [m]} \varrho(z_i) )$ for $\chi$ and $\varrho$ rotation equivariant. 
\end{prop}
The technical proof is given in Section~\ref{sec:nogoPN} in the supplementary material. An idea for a design is instead given by the following theorem.
\begin{theo} \label{th:denseset}
    The set of functions on the form \begin{align}
    f(Z) =  \sum_{i \in [m] } \gamma(\tau_i^*Z)z_i, \label{eq:repr}\end{align} where $\gamma$ is an arbitrary continuous, rotation invariant and $\Stab(0)$-invariant function,
    is dense in $\calR(m)$.
\end{theo}
We remind the reader that $\tau_i$ is the transposition of $0$ and $i$. The proof of Theorem~\ref{th:denseset}, which rests upon the density of polynomials and algebraic manipulations of them, is found  in Section \ref{sec:dense} in the supplementary material. Let us here instead concentrate on intuitively explaining it.
 
It is fruitful to interpret the values $(\gamma(\tau_i^*Z))_{i \in [m]}$ as scaled rotations $c_i\theta_i$, with $c_i \in \R$ and $\theta_i \in \S$. Considering this, \eqref{eq:repr} can be interpreted as a \emph{weighted centroid} of the point cloud, where each point can be individually rotated prior to calculation of the weighted centroid.

To calculate the rotation invariant weight $\gamma(\tau_i^*Z)$ for point $z_i$, we are allowed to examine the entire cloud, and not only $z_i$. Hence, \eqref{eq:repr} can be interpreted as an attention mechanism 
(compare, e.g., \cite{sun-cvpr-2020, fuchs2020se3transformers, xuShowAttendTell2015, xieAttentionalShapeContextNetPoint2018, vaswaniAttentionAllYou2017, jaderberg2015,lee2019set})
-- when calculating `its' weight, $z_i$ can attend to all other points in the network. It does not however do so in an arbitrary fashion: when calculating $\gamma(Z)$, because of the $\Stab(0)$-invariance, the point $z_0$ takes a special role, but the collective $(z_i)_{i \geq 1}$ is treated like a set. In the vector $\tau_i^*Z$, the special, first, position is occupied by $z_i$. Hence, when $z_i$ calculate `its' weight, it is allowed to attend to its own position $z_i$ and to the positions of the rest of the points  $(z_j)_{j \neq i}$ as a set. Finally, note that the weight calculation function $\gamma$ is shared by all the points.

\subsection{A universal architecture for \texorpdfstring{$\calR(m)$}{R(m)}} 

\begin{figure*}
\centering
\fbox{\includegraphics[width=0.45\textwidth]{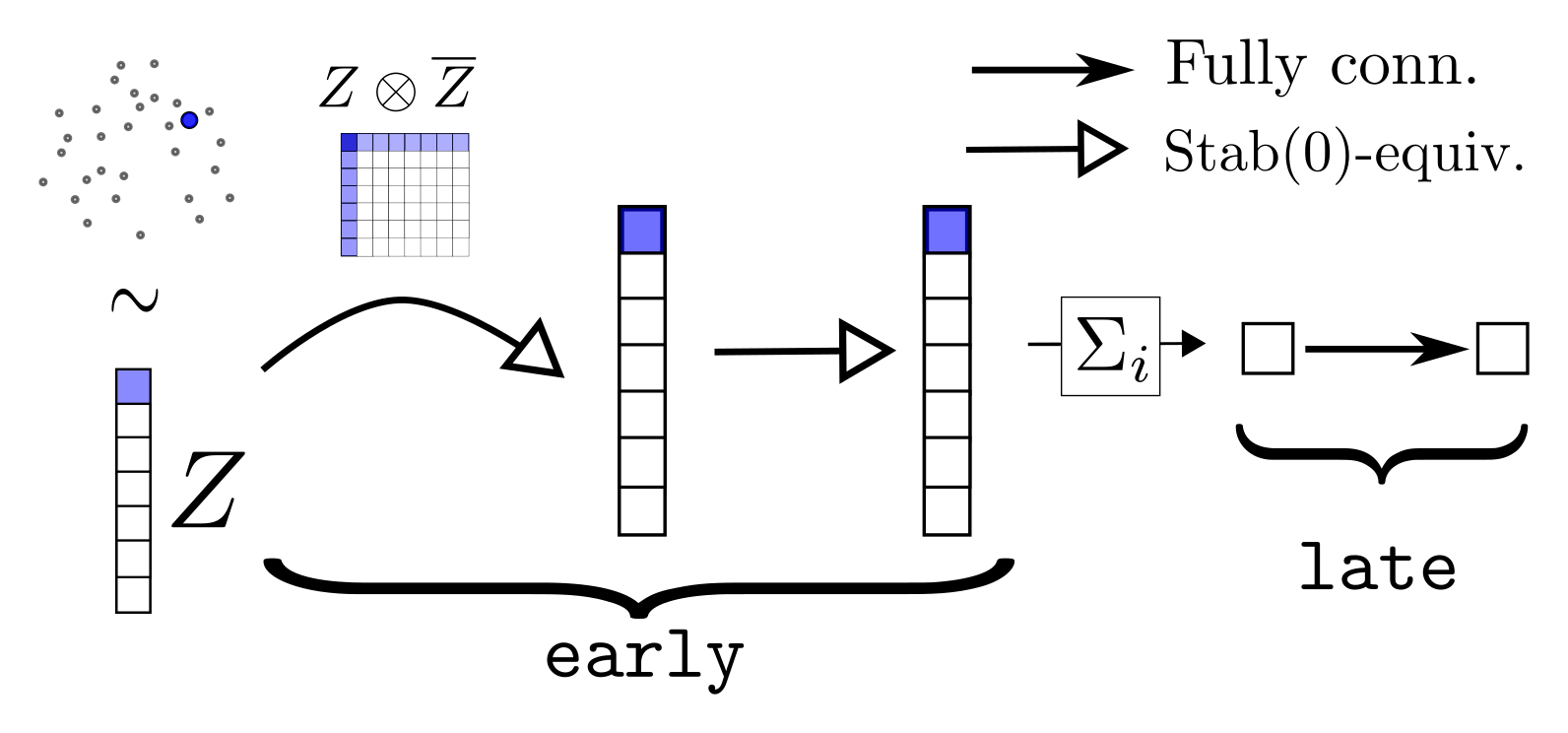}} \quad  \fbox{\includegraphics[width=0.45\textwidth]{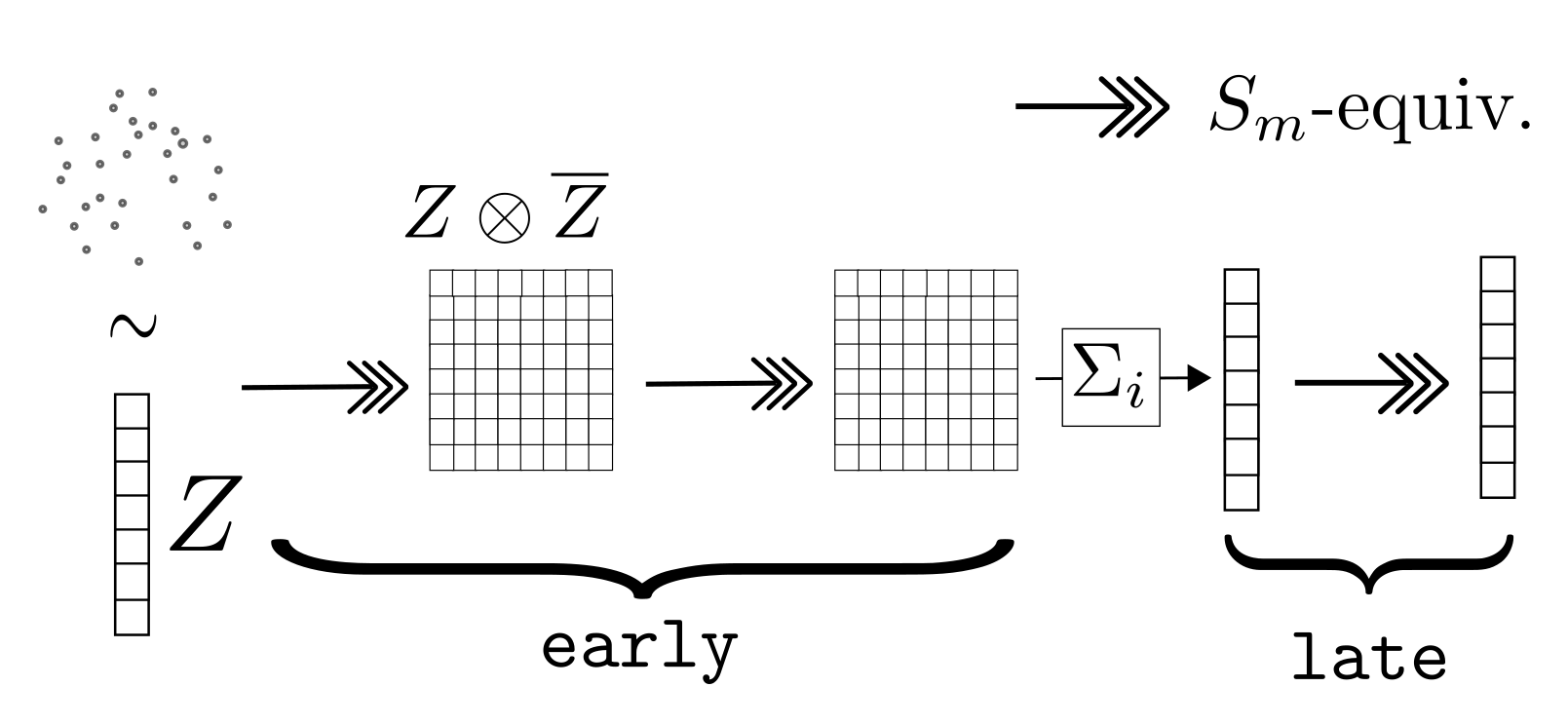}}
\caption{The $\calNS(m)$ (left) and $\calNS^+(m)$ (right) architectures. Note that one of the points take on a special role in the left architecture, whence the highlightings. \label{fig:NS} }
\end{figure*}
We now describe how a neural network for approximation of functions in $\calR(m)$ can be built. In the light of Theorem \ref{th:denseset}, we should design a weight unit $\alpha: \C^m \to \C$ which is invariant to both rotations and $\Stab(0)$-permutations, approximating the function $\gamma$. As for the rotation invariance, we propose to let the network simply act on the tensor $Z \otimes \overline{Z}=(z_i \overline{z}_j)_{i,j\in [m]}$
instead of $Z$ -- since $Z \otimes \overline{Z}$ is invariant to rotations of the network, the entire network will then automatically also be.
Note that the real part of $Z\otimes \overline{Z}$ is the Gram matrix of scalar products $(\sprod{z_i,z_j})_{i,j \in [m]}$, where we see the $z_i$ as vectors in $\R^2$.  This strategy hence has clear connections to \cite{Xu_2021_ICCV}, which uses \emph{sorted} Gram matrices of local neighbourhoods.
Compared to them, we apply a different way of handling the $\Stab(0)$-invariance. We follow a canonical design idea for equivariant networks -- first alternately apply equivariant linear layers and pointwise nonlinearities, add an invarizing step, and thereafter apply fully connected layers.
We denote the resulting set of neural networks $\calNS(m)$.\footnote{
$\cal{N}$ for \emph{network} and $\mathcal S$ for \emph{stabilizer}.}
In the following closer description, `linear layer' always refer to a real-linear layer with bias term.

The $\calNS(m)$ architecture is constructed as follows (cf.\ Figure~\ref{fig:NS}):

{\bf Early layers.} The very first layer consists of applying a $\Stab(0)$-equivariant linear layer $$B_0: (\C^m)^{\otimes 2} \to (\C^m)^{\ell_1}$$ to $Z\otimes \overline{Z}$. \textcolor{black}{Here, as in the following, $\ell_j$ refers to the number of channels in layer $j$.} Then, a nonlinearity $\rho: \C \to \C$ is applied pointwise, i.e., $\rho(X)_i= \rho(x_i)$. Concretely, we use a standard activation function separately applied to real- and imaginary parts.
 
 Subsequently, $L$ $\Stab(0)$-equivariant layers $$B_i~:~(\C^m)^{\ell_i}\to (\C^m)^{\ell_{i+1}}$$ are applied in alternation with a pointwise linearity $\rho: \C \to \C$. The final output of the early layers is a multivector $V\in (\C^m)^{\ell_L}$

 {\bf Invarization step.} Next, we calculate $v=~~\sum_{i \in [m]}v_i~$. Note that this transforms the $\Stab(0)$-equivariant multivector $V$ into a $\Stab(0)$-invariant multiscalar $v$. In fact, we could here instead apply any $\Stab(0)$-invariant functional, but we concentrate on summation for simplicity.

{\bf Late layers.} Finally, a number of fully connected layers are applied to $v$.

\

 Importantly, the very first linear layer maps into a space of multivectors, rather than multitensors. This saves a significant amount of memory compared to letting all early layers handle multitensors, which would be the naive way to process the tensor $Z\otimes \overline{Z}$. In fact, it is even possible to apply the first layer without explicitly calculating $Z\otimes \overline{Z}$ -- see Section~\ref{sec:spanning sets} of the supplementary material.

When implementing $\calNS(m)$, one of course needs a way to parametrize the $\Stab(0)$-equivariant linear layers. In Section \ref{sec:linearlayers} of the supplementary material, heavily building on the results from \cite{maron2018invariant} about permutation equivariant linear maps, we provide such a parametrization. It is not needed to know this construction in order to follow the rest of the paper. Let us just note that the number of parameters needed to describe each input-output-channel pair is independent of $m$ (just as for the permutation invariant layers in \cite{maron2018invariant}).  \newline 

In order to build a provably universal architecture for $\calR(m)$, it turns out that the above approximation of the $\gamma$-function is not enough. We instead need to add another component, a  `vector unit' $\psi: \C \to \C$ acting on the individual points $z_i$. These units use fully connected \emph{complex-linear} linearities \emph{without bias} and complex $\mathrm{ReLU}$s $\rho_\C$,
\begin{align*}
    \rho_\C(z,\eta) = \mathrm{ReLU}(\abs{z}-\eta)\tfrac{z}{\abs{z}}
\end{align*}
as nonlinearities. Here, $\eta \in \R_+$ is a learnable parameter, and $\mathrm{ReLU}$ is the real ReLU. Note that $\rho_\C$ is rotation equivariant. Since the complex-linear maps also are, the entire $\psi$-unit automatically becomes rotation equivariant.
Let us call the set of such rotation equivariant networks $\calNC$.\footnote{
$\calN$ for \emph{network} and $\calC$ for \emph{complex}. }

Using $\alpha$-units from $\mathcal{NS}(m)$ and $\psi$-units in $\calNC$, we may now build a set $\calNR(m)$\footnote{$\calN$ for \emph{network} and $\calR$ for \emph{rotation}.} of rotation equivariant, permutation invariant $\Psi$ networks through
\begin{align} \label{eq:NR}
    \Psi(Z) = \sum_{i \in [m]} \alpha(\tau_i^*Z)\psi(z_i).
\end{align}
Our main result is now that this architecture is universal for $\calR(m)$.

\begin{theo}\label{th:universality}
    $\calNR(m)$ is universal for $\calR(m)$.
\end{theo}
\begin{proof}[Proof Sketch] The entire proof is too long to present here, and is postponed to Section \ref{sec:universal} of the supplementary material. Let's however sketch it.

{\bf Step 1: Universality of $\calNS(m)$}. First, one proves that for any $\epsilon>0$, $\calNS(m)$ is dense in the set of $\Stab(0)$- and rotation invariant function \emph{when restricting to point clouds with} $\abs{z_0}> \epsilon$. Intuitively, we apply the  Stone-Weierstrass Theorem \cite[7.32]{rudin} to show that $\alpha$ can approximate any function of the form $\phi(\abs{z_0}^2,z_0\overline{Z})$, where $\phi$ is permutation invariant with respect to the second argument. Since we are only concerned with the case of $\abs{z_0}\neq 0$, the map $Z \to (z_0,z_0\overline{Z})$ is injective. From that, we obtain the claim.

{\bf Step 2: Universality of $\calN\calR(m)$.} The first step shows that for any fixed $\epsilon>0$, $\alpha$ can be chosen so that $\alpha(\tau_i^*Z) \approx \gamma(\tau_i^*Z)$ as long as $\abs{z_i}>\epsilon$. However, since the product $\gamma(\tau_i^*Z)\cdot z_i$ is small if $\abs{z_i}<\epsilon$, we can still achieve a good approximation anywhere. 
This is the technical reason for the inclusion of the  vector unit -- it can eliminate any problem with large  $\alpha(\tau_i^*Z)$-values when $z_i$ is small.
\end{proof}

\section{Modifications of the universal architecture}

Although the architecture in the previous section is universal, we will modify it in a number of ways  before using them for our experiments.
\subsection{A richer, parallel architecture}

In the $\calNR(m)$-nets, note that each permuted version $\tau_i^*Z$ of the cloud is sent through the $\alpha$-unit individually. It would intuitively be better to calculate all weights in parallel, and in that process let the weight values `communicate' with each other. A simple way to achieve this is the following modification, 
which we denote $\calNS^+(m)$.

\underline{The $\calNS^+(m)$ architecture} consists of the following: (see also Figure~\ref{fig:NS}).

{\bf Early layers.} Apply an $S_m$-equivariant linear layer $$B_0^+~:~(\C^m)^{\otimes 2}\to((\C^m)^{\otimes 2})^{\ell_1}$$ to $Z \otimes \overline{Z}$. Subsequently, apply, in alternation, a point-wise non-linearity and $S_m$-equivariant layers $$B_i^+~:~((\C^m)^{\otimes 2})^{\ell_i}\to((\C^m)^{\otimes 2})^{\ell_{i+1}}.$$ The final output of the early layers is then a multitensor $T\in((\C^m)^{\otimes 2})^{\ell_L}$.

{\bf Invarization step.} Next, $V~=~(\sum_{j \in [m]}T_{ij})_{i \in [m]}$ is calculated, which transforms the $S_m$-equivariant multi{-}tensor $T$ to an $S_m$-equivariant multivector $V$.

{\bf Late layers.} Now apply, in alternation, $S_m$-equivariant layers 
  $ C_i : (\C^m)^{\ell_{L+i}} \to (\C^m)^{\ell_{L+i+1}}$
and pointwise non-linearities. The final network output is 
$\alpha^+(Z)\in \C^m$.

\

We also modify the architecture for calculating the $\psi$-units: We still use $\rho_\C$ as the non-linearity and apply $\C$-linear layers, however $S_m$-equivariant such to the entire cloud $Z$. The final output of such networks is thus a vector $\psi^+(Z)\in \C^m$. The set of these networks are called $\calNC^+$.

Given an $\alpha^+\in \calNS^+(m)$ and a $\psi^+ \in \calNC^+$, we now build a network $\Psi^+$ through
\begin{align*} 
    \Psi^{+}(Z)  = \sum_{i \in [m]} \alpha^+(Z)_i \cdot \psi^+(Z)_i
\end{align*}
Let us denote the set of these networks $\mathcal{NR}^+(m)$. These networks are still equivariant, and are at least as expressive as the non-modified ones. 
\begin{prop}\label{prop:NRplus}
(i) The new architecture has the correct equivariance, i.e., $\mathcal{NR}^+(m) \sse \mathcal{R}(m)$. \\
(ii) The new architecture is at least as expressive as the non-modified, i.e., $\mathcal{NR}(m) \sse \mathcal{NR}^+(m)$.
\end{prop}
See Section \ref{sec:Nrplus} in the supplement for a proof.

 It does take more parameters to parametrize each input-output-channel pair of the linear layers in the $\calNR^+(m)$, but this can be compensated by using less input-output-channel pairs. As for the memory requirements, we have to handle 2-tensors in memory, which leads to a quadratic cost. This is worse than the $\calNR(m)$-architecture, whose memory cost is only linear. However, recall that we need to calculate $m$ values  $\alpha(\tau_i^*Z)$, $i \in [m]$ for each application of the network. If we want to parallelize those calculations, which we should do for efficiency,  we need to handle $m$ vectors, again resulting in a quadratic memory cost.

A subtle, but nonetheless reasonable, reason for using  the $\mathcal{NR}^+(m)$-architecture instead of the $\calNR(m)$ architecture is that it allows for more exchange of information between the points. As an example, notice that when calculating the weight $\alpha(\tau_i^*Z)$, each early layer in an $\calNR(m)$-net is only allowed to attend to one vector, which can be seen as a preliminary version of the vector weight. In the $\calNR(m)^+$-architecture, it is additionally allowed to attend to all the `preliminary weight vectors', i.e., the other columns of the input tensor (as a set). This arguably makes the modified architecture more versatile.

\section{Approximating functions in \texorpdfstring{$\calR_2(m)$}{R\_2(m)}}\label{sec:NR2}
In our experiments, we will actually consider tasks which take \emph{pairs} $(Z,X)$ of point clouds as indata. Thereby, we assume that for each $i$, the  points $z_i$ and $x_i$ correspond to each other, meaning that we only get invariance towards \emph{simultaneous permutations of both clouds}. The tasks we consider will be (or will be transformed into ones that are)  rotation equivariant with respect to one cloud, and rotation invariant with respect to the other. That is, we will have to approximate functions $f$  such that for every $\pi \in S_m$ and $\theta, \omega \in \mathbb{S}$, we have
\begin{align*}
    f(\theta \pi^*Z, \omega \pi^* X) = \theta f(Z,X).
\end{align*}
We denote the set of such functions $\calR_2(m)$.

We can use the same ideas as above to build an architecture for them.
We propose to use the exact same scheme, with the only difference that the very first layer $L$ of the  $\alpha$-unit depends on $Z\otimes \overline{Z}$ and $X \otimes \overline{X}$, as 
$$L(Z,X) = A(Z \otimes \overline{Z}) +B(X \otimes \overline{X}),$$
where $A$ and $B$ are linear layers of the same flavor as for $\mathcal{NR}(m)$ and $\mathcal{NR}(m)^{+}$, respectively. This yields architectures $\calNR_2(m)$ and $\calNR_2^+(m)$. In Section~\ref{sec:twoclouds} of the supplementary material, we prove that $\calNR_2(m)$ is  \emph{not} dense for the whole of $\calR_2(m)$. We however also prove that if we only consider cloud pairs $(Z,X)$ for which no points close to the origin in $X$ correspond to points far away from the origin in $Z$, we again obtain universality for both versions.

\subsection{A deeper architecture}

\begin{figure}
\centering
\fbox{ \includegraphics[width=.35\textwidth]{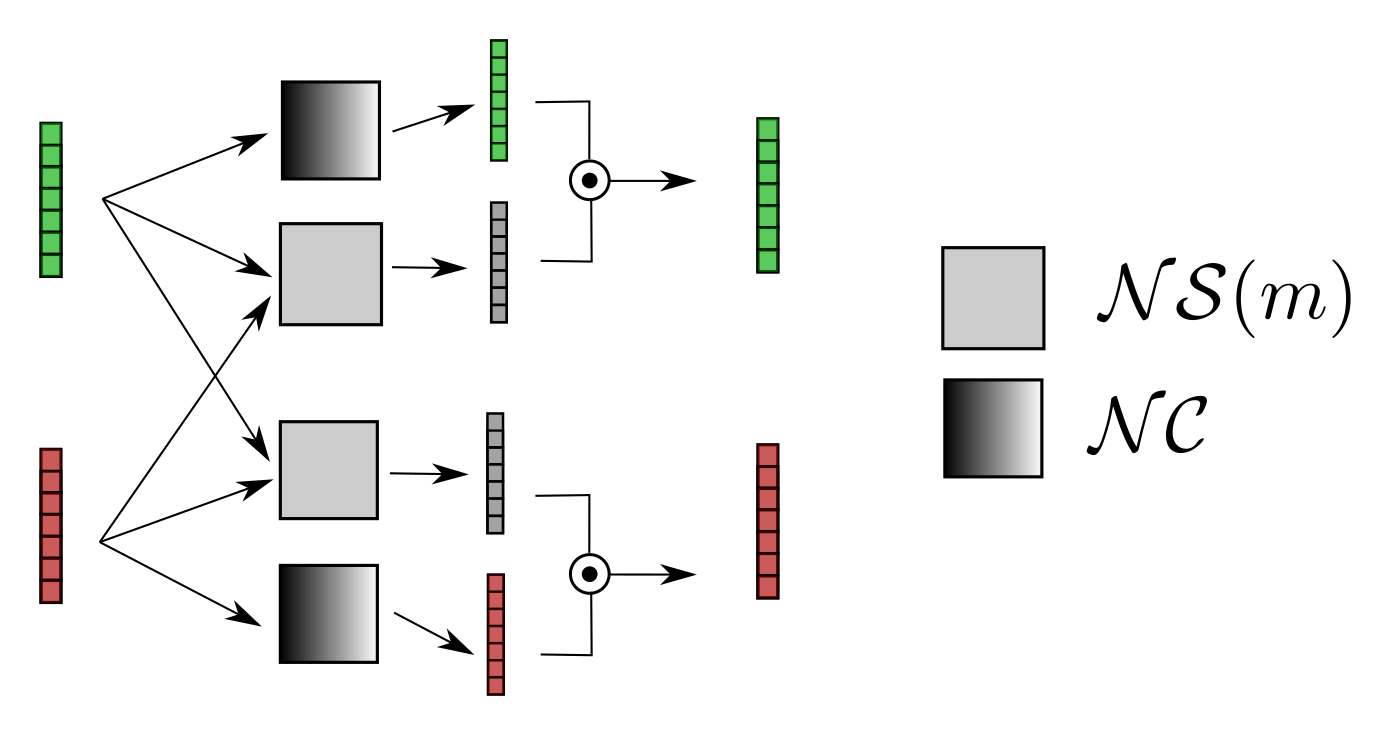}} 
 \caption{The architecture of a ZZ-unit. \textcolor{black}{Two clouds are fed into $SO(2)$-invariant weight units and $SO(2)$-equivariant vector units, and then combined to produce a new pair of clouds.} Best viewed in color. \label{fig:Rt2}}
\end{figure}

We can easily combine several weight and vector units $\alpha_k\in \calNS^+(m)$, $\psi_k \in \calNC^+$, to build an iterative architecture.  If $Z=Z^0$ is the input cloud, we iteratively define new clouds $Z^k$ through
\begin{align*}
 z^{k+1}_i &=\alpha_k^+(Z^{k})_i \cdot \psi_k^+(Z^{k})_i, \quad  
\end{align*}
$i\in [m]$. A particular case where such chains of units can be especially beneficial is the case when the cloud is filled with outliers. The weight units of early layers can then be used to filter those out, by giving the outliers small weights. They will then cluster around the origin, which can safely be ignored by later weight units. This is in spirit similar to (attentive) context normalization \cite{moo-cvpr-2018,sun-cvpr-2020}.

In the cloud pair case, we can iteratively construct new pairs of clouds by chaining pairs of weight and vector units (see Figure \ref{fig:Rt2}):
\begin{align*}
    z^{k+1}_i &=\alpha_k^+(Z^{k},X^{k})_i \cdot \psi_k^+(Z^{k})_i, \\ x^{k+1}_i &=\beta_k^+(X^{k},Z^{k})_i \cdot \phi_k^+(X^{k})_i.
\end{align*}
The final output of such a network is then a pair of scalars $(F_0(Z,X),F_1(X,Z))$, where the first scalar is equivariant to rotations in the first cloud, and invariant to rotations in the second, and vice versa. If we let $\alpha_k=\beta_k$ and $\psi_k =\phi_k$, we will even obtain a network which is equivariant to switching the pairs. This is the version 
we are using in our experiments. Since the weight-units are using tensors of the form $Z\otimes \overline{Z}$ as input, we will refer to such layers as \emph{ZZ-units}.

To obtain a rotation equivariant output of the network, we sum over $i$ in the final (respective) cloud. The set of such obtained architecture will be referred to as \emph{ZZ-nets}.

\subsection{Limitations of the architecture}

Although our architecture is provably universal, it has its limitations. 
First and foremost, it operates on tensors rather than vectors, making its memory requirement quadratic in the number of points per cloud. Secondly, all linear layers of our architecture are global in nature, which could hurt performance.

A simple way to mitigate these issues would be to let the weight units $\alpha$ only operate on the nearest neighbors to $z_i$ when calculating the weight for $i$ -- we would then return to a memory requirement which is linear in the cloud sizes, and induce locality. \textcolor{black}{However, such an architecture would not be universal.}

\section{Experiments}

Here we present two experiments to demonstrate our network in action.
Further details about the experiments are given in Section~\ref{sec:exp} of the supplementary material. Code for the experiments is available at \href{https://github.com/georg-bn/zz-net}{github.com/georg-bn/zz-net}.

\subsection{Estimating rotations between noisy point clouds}
Let us, as a proof of concept more than anything else, test our model on a toy problem: Given a point pair $(Z,X)$, estimate a rotation $R(Z,X)$ so that  $X=R(Z,X)Z$. This rotation responds to rotations of either cloud through $R(\theta Z, \omega X) = \omega \overline{\theta} R(Z,X)$. 

If $Z$ and $X$ are completely noise-free, this is of course trivial (one can e.g.\ calculate $\sfrac{z_0}{x_0}$). In order to make the problem more challenging, we consider a setting with both inlier and outlier noise.

{\bf Data.} We synthetically generate data. The details of the data generation are presented in the supplementary~\ref{sec:data_gen}. Each point cloud pair $(Z,X)$ contains $m=100$ correspondences out of which a fraction $r$ are outliers. The inliers lie on a triangle with low-level inlier noise. An example pair is shown in Figure~\ref{fig:cloudpair}.

\begin{figure}
    \centering
    \includegraphics[width=.3\textwidth]{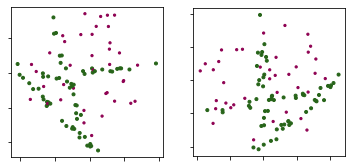}
    \caption{A pair of noisy point clouds as used in the rotation estimation experiments. The inlier points are larger and colored green for illustration purposes. Here the outlier ratio $r$ is 0.4. }
    \label{fig:cloudpair}
\end{figure}

{\bf Models.} We test two versions of our model: A `broad' and a `deep' one. The `broad' model consists of a single ZZ-unit, with $2$ early and $3$ late layers in the weight unit, and $2$ layers in the vector unit. The `deep' model consists of three ZZ-units, where each unit only has $1$ early and $2$ late weight-layer units, and $1$ vector layer, and each such layer is smaller than for the broad model. The broad unit has around $4$k, and the deep around $7$k, parameters in total.

We train a unit with weights shared, thus outputting two scalars $F(Z,X)$ and $F(X,Z)$. The final output of our model $\widehat{\theta}(X,Z) = F(X,Z)\overline{F(Z,X)} \in \C$ then responds correctly to rotations of either cloud.

For comparison, we implement two alternative  models. A PointNet and a simplified version of ACNe \cite{sun-cvpr-2020} which we call `ACNe\textminus'. They have $34$k and $11$k parameters respectively. Details about these models are presented in the supplementary~\ref{sec:comp_mod}.

{\bf Experiments.} We test each of the models on four outlier ratios: $0.4,0.6,0.8$ and $0.85$. We use an $\ell_2$-loss between the ground truth rotation and the output of the networks, and manually tune hyperparameters to optimize the mean error on the validation set. 

To evaluate the experiments, we test how many of the ground truth rotations the models can estimate within an error that corresponds to a difference $1^\circ$, $5^\circ$ and $10^\circ$ for two normalized complex numbers (note that the output of our models is not necessarily normalized), respectively. The results are presented in Table~\ref{tab:RotResults}. The broad model easily beats the PointNet model, and also the `ACNe\textminus'-model for low outlier ratios, but starts to struggle against the context-normalization based model for $r=0.85$. The deep model however easily outperforms all other models. 

We notice that some models struggled somewhat on the $r=0.8$-data set. We had to stop the broad model early due to severe overfitting, and the `ACNe\textminus' model did worse on the $0.8$-set than on the $0.85$-set. We suspect that this ultimately boils down to the fact that due to our data generation method, the actual outlier ratios are random. Therefore, the $0.8$ dataset could contain some especially hard examples just by chance.

\begin{table}
\begin{tabular}{|r || c | c | c | }
\hline
Outlier ratio & \multicolumn{3}{c|}{$r=0.4$} \\
\hline
     Threshold & $1^\circ$ & $5^\circ$ & $10^\circ$  \\
   \hline Broad ZZ-net  &   .42  & .97  & .99 \\
   Deep ZZ-net      & {\bf.85}     &  {\bf .99}  & {\bf 1.0} \\
   PointNet  & .02	& .45 & .78 \\
   ACNe\textminus& .05 &	.63 & 	.96
   \\ \hline
   \hline
Outlier ratio & \multicolumn{3}{c|}{$r=0.8$} \\
\hline
     Threshold & $1^\circ$ & $5^\circ$ & $10^\circ$  \\
   \hline Broad ZZ-net${}^\dagger$  &   .03  & .46   & .81 \\
   Deep ZZ-net     &    {\bf .32} & {\bf .90}   & {\bf.96} \\
   PointNet  & .03	& .25 & 	.54 \\
   {ACNe\textminus}  & .01	& .27 &	.69 \\ \hline
\end{tabular}
\begin{tabular}{| c | c | c | }
\hline
\multicolumn{3}{|c|}{$r=0.6$} \\
 \hline
 $1^\circ$ & $5^\circ$ & $10^\circ$  \\
   \hline    .21    & .87  & .96 \\
            {\bf.84}  &  {\bf .99}  & {\bf .99} \\
           .03 & 	.34	& .67
\\
      .04	& .54 &	.90
   \\ \hline
 \hline
\multicolumn{3}{|c|}{$r=0.85$} \\
\hline
 $1^\circ$ & $5^\circ$ & $10^\circ$  \\
   \hline     .02   & .24  & .50 \\
             {\bf.11}   & {\bf.73}  & {\bf.90} \\
      .03 & 	.21 & 	.37 \\
     .02 &	.45 &	.75 \\
   \hline
\end{tabular}
\caption{Results for rotation estimation with varying outlier ratios. ${}^\dagger$This experiment was stopped early due to severe overfitting. \label{tab:RotResults}
}
\end{table}

\subsection{Essential matrix estimation} \label{sec:exp_ess}
The input in the problem is a (noisy) set of calibrated 2D-2D correspondences
$\{(p_1, p_2)\}$, where $p_1, p_2\in \R^2$ are points of interest in two images of the same object. 
The task is then to estimate the essential matrix $E\in \R^{3\times 3}$ such that 
$\tilde p_2^T E \tilde p_1 = 0$ for the (correct) correspondences.
Here, $\tilde p$ is the homogeneous representation of $p$ obtained by adding a third coordinate $1$ to $p$.
See \cite{hartleyMultipleViewGeometry2003} for an in depth description of essential matrices.
Considering the points $\{p_1\}$ as elements of $\C$ and stacking them into a vector
yields the $Z$ vector considered in earlier sections, and similar for $\{p_2\}$ and $X$.

\textbf{Rotation equivariance of $E$.} 
If $\tilde p_2^T E \tilde p_1 = 0$ for a set of correspondences $\{(p_1, p_2)\}$,
it follows that if we rotate $p_1$ by an image plane rotation $R\in SO(2)$, say to $q_1 = Rp_1$, then 
$\tilde p_2^T E \tilde R^T \tilde q_1 = 0$ where $\tilde R \in SO(3)$ is the rotation obtained by applying $R$ as a rotation around the $z$-axis. 
 Hence, $ E \tilde R^T$ is an essential matrix for the correspondences $\{(q_1, p_2)\}$.
Similarly one shows that a rotation of $p_2$ to $q_2 = Rp_2$ yields an essential matrix
$\tilde R E$ for the correspondences $\{(p_1, q_2)\}$.

The essential matrix has an SVD of the form $E = USV^T$, where $U$ and $V$ are orthogonal
and $S=\mathrm{diag}(1, 1, 0)$.
Since $E$ is only determined up to scale, we can choose $U$ and $V$ as rotation matrices.
It is then possible to decompose $U$ and $V$ into Euler rotations about the $z$- and $y$-axes:
$E = R_{z,2} R_{y,2} R_{z',2} S R_{z',1}^T R_{y,1}^T R_{z,1}^T$ where we can merge $R_{z',2}$ and 
$R_{z',1}$ as they commute with $S$.
We obtain $E = R_{z,2} R_{y,2} R_{z'} S R_{y,1}^T R_{z,1}^T$
and we have one degree of freedom for each $R$, thus five in total, as expected. 

The equivariance properties of $E$ imply that $R_{z,1}$ is equivariant to rotations in $p_1$ and $R_{z,2}$ is equivariant to rotations in $p_2$, both while being invariant to rotations of the other cloud. The other matrices are invariant to rotations in both clouds. We design the network to output
five complex numbers on the unit circle $\mathbb S$, where two of them lie in $\calR_2(m)$, and three are invariant to rotations in either cloud.

\textbf{The model: ZZ-net.}
We use a back-bone architecture $\mathcal{B}$ with three ZZ-units, the first two units having 
2 early, 2 late and 2 vector layers and the last unit having 1 early, 1 late and 1 vector layer.
We also add skip-connections between the units in the back-bone for ease of training.
This back-bone outputs 8 channels of point clouds which are fed into two further units.
One is a ZZ-unit $\mathcal{E}$ which is responsible for predicting the \emph{equivariant}
$R_{z,2}$ and $R_{z,1}$. The second is a PointNet $\mathcal{I}$ that takes as input the
$\alpha^+$-values of the last layer of the back-bone (which are rotation invariant) to predict
the \emph{invariant} $R_{y,2}$, $R_{y,1}$ and $R_{z'}$. 

To account for the symmetry of changing order of the clouds, we approximate $R_{z,1}$ with $\mathcal{E}(\mathcal{B}(Z,X))$, and $R_{z,2}$ with $\mathcal{E}(\mathcal{B}(X,Z))$. In turn,
$\mathcal{I}(\mathcal{B}(Z,X))$ yields two rotations: $R_{y,1}$ and $R_{z',1}$, while
$\mathcal{I}(\mathcal{B}(X,Z))$ yields $R_{y,2}$ and $R_{z',2}$. $R_{z',1}$ and $R_{z',2}$ are
combined to form $R_{z'} = R_{z',2} R_{z',1}^T$.  In total, the architecture thus outputs five rotations. It has around 55k parameters.

Similar to OANet~\cite{zhang2019oanet}, we use a geometric loss based on virtual matches generated from the ground truth essential matrix.
For further information on the model and training setup, see the supplementary~\ref{sec:ess_train_setup}.

\textbf{Data.} We use the subset of the YFCC100M data \cite{thomeeYFCC100MNewData2016}
corresponding to the sequence `Reichstag' compiled by \cite{heinlyReconstructingWorldSix2015}.
Two example images can be seen in Figure~\ref{fig:reichstag}.
The image sequence is processed to obtain SIFT-matches \cite{loweDistinctiveImageFeatures2004}
between image pairs using code supplied by the authors of CNe \cite{moo-cvpr-2018}.
Some image pairs are discarded due to visibility issues and for each remaining image pair
2000 correspondences are found, many of which might be incorrect matches.
The obtained dataset is quite small -- the training set consists of 3302,
the validation set of 56 and the test set of 52 point cloud pairs\footnote{
In fact half of the 3302 (resp. 56, 52) pairs correspond to the 
other half but with the two images in the pair swapped.}.
Therefore our experiments should be viewed as a limited data case study.

\begin{figure}
    \centering
    \includegraphics[width=0.8\columnwidth]{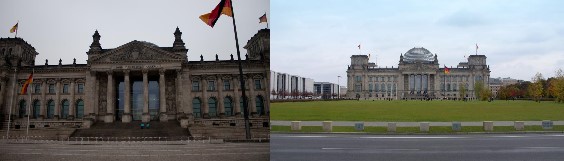}
    \caption{Two images from the `Reichstag' data.}
    \label{fig:reichstag}
\end{figure}

\textbf{Evaluation metric.}
From the essential matrix we can recover the rotation between the two views and the 
translation between the views up to scale. 
We evaluate the estimated
essential matrix in terms of the mAP score
proposed by \cite{moo-cvpr-2018}, which is a measure of error in angle of the estimated translation
and rotation axes.

\textbf{Comparisons.}
We compare against CNe~\cite{moo-cvpr-2018}, \\
OANet~\cite{zhang2019oanet} and ACNe~\cite{sun-cvpr-2020}.
These methods build on the idea of learning inlier weights for the correspondences
and using a weighted formulation of the 
8 point method~\cite{longuet-higginsComputerAlgorithmReconstructing1981}
as a final layer in the network. They are all
very good at handling outliers, as they are explicitly trained on classifying each
correspondence as an inlier or outlier as well as outputting a reasonable essential matrix.
In contrast, our network is only trained to output a reasonable essential matrix
but does it in a way that is resilient to rotations of the data, which is not
part of the other frameworks.
We do not compare against 
T-net~\cite{Zhong_2021_ICCV} as they have not published their code at the time of writing.

We retrain the implementation of the authors of CNe, OANet and ACNe on the `Reichstag' dataset.
For the sake of fairness, we do not use RANSAC at test time. Note that therefore our reported numbers
for CNe are below what they report in their paper.
CNe has 394k parameters, ACNe 400k parameters and OANet 2347k parameters.

\textbf{Rotated test data.}
To demonstrate the resilience of our method to rotation perturbations
of the data, we evaluate both on the original test data as well as versions of the
test data where the $p_1$ points are rotated a random amount (and the ground truth essential
matrix is altered correspondingly, as described earlier). We sample rotations for each test example
uniformly in the interval $(-a, a)$ and consider three different values for the 
maximum rotation angle: $a=\ang{30},\ang{60},\ang{180}$.
All methods are evaluated on the same rotated versions of the test set for consistency. 

\textbf{Results.}
We present results in Table~\ref{tab:essential20} for mAP at $\ang{20}$.
The results for our method are averaged over two training runs.
The maximum difference in mAP scores between
the two runs was 0.01. mAP scores at $\ang{10}$ and $\ang{30}$ are presented in the supplementary~\ref{sec:ess_further} and they tell a similar story.
\begin{table}
\begin{tabular}{|r || c | c | c | c |}
\hline
     Max. test rot. $a=$& $\ang{0}$ & $\ang{30}$ & $\ang{60}$  & $\ang{180}$ \\
   \hline ZZ-net (Ours)  & 0.26 & \textbf{0.26} & \textbf{0.26} & \textbf{0.26} \\
   ACNe         & \textbf{0.67} & 0.25 & 0.15 & 0.038 \\
   CNe          & 0.43 & 0.14 & 0.12 & 0.0048 \\
   OANet        & 0.42 & 0.24 & 0.077 & 0.0048 
   \\ \hline
\end{tabular}
\caption{Results for essential matrix estimation. mAP at $\ang{20}$ error in the estimated translation and rotation vectors for different values of
image plane rotations $a$ at test time.
\label{tab:essential20}}
\end{table}

\textbf{Discussion.}
Our method does not compete well on the base problem ($a=\ang{0}$).
This may in part be due to the order of magnitude fewer parameters of our network.
Note that we had to limit the number of parameters due to the quadratic memory cost of
the weight-units. We however demonstrate the resilience to rotation perturbations of \mbox{ZZ-net}.
Already at modest rotations uniformly sampled from $-\ang{30}$ to $\ang{30}$ it is on par
with the more mature competitors. At larger rotations \mbox{ZZ-net} is superior. It should however be noted that for this dataset, all images are oriented close to parallel with the ground. There is hence a clear bias in the training data, so that the comparison to the other models on artificially rotated test data is not completely fair.

We still believe that rotation equivariance can add robustness to methods attacking the essential matrix
estimation problem and regard it as an interesting future research direction to try to merge our approach
with the outlier robust methods, using for instance the weighted 8-point method.
Furthermore, it would be interesting
to develop methods which are equivariant only to small rotations -- rotations larger
than $\ang{60}$ will typically not be seen in practice. This would require leaving the mathematical framework of group theory, as such bounded rotations do not form a group.

\section{Conclusions}
We have presented a foundational framework for learning tasks based on a rotation equivariant and permutation invariant neural network architecture. A proof is given showing that this architecture is indeed universal. We have described several ways of modifying the architecture, in particular, how to extend it to pairs of point clouds as appearing in correspondence problems and how to perform efficient computations. As for limitations, the framework is only applicable in two dimensions. 
Our architecture further lacks locality and has a high memory requirement. To mitigate the latter issues are examples of interesting future work.

\section*{Acknowledgements}
The authors acknowledge support
from CHAIR, SSF, as well as 
WASP
funded by the Knut and Alice Wallenberg Foundation. 
The computations were enabled by resources provided by SNIC at C3SE.

{\small
\bibliographystyle{ieee_fullname}
\bibliography{egbib}

\begin{thebibliography}{10}\itemsep=-1pt

\bibitem{aronssonHomogeneousVectorBundles2021}
Jimmy Aronsson.
\newblock Homogeneous vector bundles and ${{G}}$-equivariant convolutional
  neural networks.
\newblock {\em arXiv:2105.05400 [cs, math, stat]}, May 2021.

\bibitem{big_dipper_pic}
BreakdownDiode.
\newblock {Big Dipper 20210116.jpg, used under Creative Commons
  Attribution-ShareAlike 4.0 International}
  \href{https://creativecommons.org/licenses/by-sa/4.0/deed.en}{license} //
  {Stars} in main constellation brightened.
\newblock
  \url{https://commons.wikimedia.org/wiki/File:Big_Dipper_20210116.jpg}, 2021.

\bibitem{bronsteinGeometricDeepLearning2021}
Michael~M. Bronstein, Joan Bruna, Taco Cohen, and Petar Veli{\v c}kovi{\'c}.
\newblock Geometric {{Deep Learning}}: Grids, {{Groups}}, {{Graphs}},
  {{Geodesics}}, and {{Gauges}}.
\newblock {\em arXiv:2104.13478 [cs, stat]}, May 2021.

\bibitem{cohen2016group}
Taco Cohen and Max Welling.
\newblock Group equivariant convolutional networks.
\newblock In {\em Int. Conf. Machine Learning}, 2016.

\bibitem{cohenEquivariantConvolutionalNetworks2021}
Taco~S. Cohen.
\newblock {\em Equivariant Convolutional Networks ({{PhD Thesis}})}.
\newblock PhD thesis, University of Amsterdam, June 2021.

\bibitem{cybenko1989approximation}
George Cybenko.
\newblock Approximation by superpositions of a sigmoidal function.
\newblock {\em Mathematics of control, signals and systems}, 2(4):303--314,
  1989.

\bibitem{dengVectorNeuronsGeneral2021}
Congyue Deng, Or Litany, Yueqi Duan, Adrien Poulenard, Andrea Tagliasacchi, and
  Leonidas~J. Guibas.
\newblock Vector neurons: A general framework for so(3)-equivariant networks.
\newblock In {\em Proceedings of the IEEE/CVF International Conference on
  Computer Vision (ICCV)}, pages 12200--12209, October 2021.

\bibitem{dym-maron-2021}
Nadav Dym and Haggai Maron.
\newblock On the universality of rotation equivariant point cloud networks.
\newblock In {\em ICLR}, 2021.

\bibitem{eidswick1968proof}
JA Eidswick.
\newblock A proof of {N}ewton's power sum formulas.
\newblock {\em The American Mathematical Monthly}, 75(4):396--397, 1968.

\bibitem{pytorch_lightning_falcon_2019}
William Falcon and {The PyTorch Lightning team}.
\newblock {PyTorch} {Lightning}, Mar. 2019.

\bibitem{finziPracticalMethodConstructing2021}
Marc Finzi, Max Welling, and Andrew~Gordon Wilson.
\newblock A {{Practical Method}} for {{Constructing Equivariant Multilayer
  Perceptrons}} for {{Arbitrary Matrix Groups}}.
\newblock {\em arXiv:2104.09459 [cs, math, stat]}, Apr. 2021.

\bibitem{finzi2021icml}
Marc Finzi, Max Welling, and Andrew Gordon~Gordon Wilson.
\newblock A practical method for constructing equivariant multilayer
  perceptrons for arbitrary matrix groups.
\newblock In Marina Meila and Tong Zhang, editors, {\em Proceedings of the 38th
  International Conference on Machine Learning}, volume 139 of {\em Proceedings
  of Machine Learning Research}, pages 3318--3328. PMLR, 18--24 Jul 2021.

\bibitem{fuchs2020se3transformers}
Fabian~B. Fuchs, Daniel~E. Worrall, Volker Fischer, and Max Welling.
\newblock {SE(3)-Transformers}: {3D} roto-translation equivariant attention
  networks.
\newblock In {\em NeurIPS}, 2020.

\bibitem{fukushima-1980}
K. Fukushima.
\newblock Neocognitron: A self-organizing neural network model for a mechanism
  of pattern recognition unaffected by shift in position.
\newblock {\em Biol. Cybernetics}, 36:193--–202, 1980.

\bibitem{gerkenGeometricDeepLearning2021}
Jan~E. Gerken, Jimmy Aronsson, Oscar Carlsson, Hampus Linander, Fredrik
  Ohlsson, Christoffer Petersson, and Daniel Persson.
\newblock Geometric {{Deep Learning}} and {{Equivariant Neural Networks}}.
\newblock {\em arXiv:2105.13926 [hep-th]}, May 2021.

\bibitem{hartleyMultipleViewGeometry2003}
Richard Hartley and Andrew Zisserman.
\newblock {\em Multiple View Geometry in Computer Vision}.
\newblock {Cambridge University Press}, {Cambridge, UK ; New York}, 2nd ed
  edition, 2003.

\bibitem{heinlyReconstructingWorldSix2015}
Jared Heinly, Johannes~L. Schonberger, Enrique Dunn, and Jan-Michael Frahm.
\newblock Reconstructing the world* in six days.
\newblock In {\em 2015 {{IEEE Conference}} on {{Computer Vision}} and {{Pattern
  Recognition}} ({{CVPR}})}, pages 3287--3295, {Boston, MA, USA}, June 2015.
  {IEEE}.

\bibitem{izadi2011kinectfusion}
Shahram Izadi, David Kim, Otmar Hilliges, David Molyneaux, Richard Newcombe,
  Pushmeet Kohli, Jamie Shotton, Steve Hodges, Dustin Freeman, Andrew Davison,
  and Andrew Fitzgibbon.
\newblock Kinectfusion: Real-time 3d reconstruction and interaction using a
  moving depth camera.
\newblock In {\em UIST '11 Proceedings of the 24th annual ACM symposium on User
  interface software and technology}, pages 559--568. ACM, October 2011.

\bibitem{jaderberg2015}
Max Jaderberg, Karen Simonyan, Andrew Zisserman, and koray {kavukcuoglu}.
\newblock Spatial transformer networks.
\newblock In C. Cortes, N. Lawrence, D. Lee, M. Sugiyama, and R. Garnett,
  editors, {\em Advances in Neural Information Processing Systems}, volume~28.
  {Curran Associates, Inc.}, 2015.

\bibitem{keriven2019universal}
Nicolas Keriven and Gabriel Peyr{\'e}.
\newblock Universal invariant and equivariant graph neural networks.
\newblock {\em NeurIPS}, 32:7092--7101, 2019.

\bibitem{adam}
Diederik~P. Kingma and Jimmy Ba.
\newblock Adam: {A} method for stochastic optimization.
\newblock In Yoshua Bengio and Yann LeCun, editors, {\em 3rd International
  Conference on Learning Representations, {ICLR} 2015, San Diego, CA, USA, May
  7-9, 2015, Conference Track Proceedings}, 2015.

\bibitem{kosmann-schwarzbachGroupsSymmetries2010}
Yvette {Kosmann-Schwarzbach}.
\newblock {\em Groups and {{Symmetries}}}.
\newblock {Springer New York}, {New York, NY}, 2010.

\bibitem{lang2021ICLR}
Leon Lang and Maurice Weiler.
\newblock A wigner-eckart theorem for group equivariant convolution kernels.
\newblock In {\em International Conference on Learning Representations}, 2021.

\bibitem{lecun-1989}
Y. LeCun, B. Boser, J.~S. Denker, D. Henderson, R.~E. Howard, W. Hubbard, and
  L.~D. Jackel.
\newblock Backpropagation applied to handwritten zip code recognition.
\newblock {\em Neural Computation}, 1(4):541--551, 1989.

\bibitem{lee2019set}
Juho Lee, Yoonho Lee, Jungtaek Kim, Adam Kosiorek, Seungjin Choi, and Yee~Whye
  Teh.
\newblock Set transformer: A framework for attention-based
  permutation-invariant neural networks.
\newblock In {\em Proceedings of the 36th International Conference on Machine
  Learning}, pages 3744--3753, 2019.

\bibitem{raytune_liaw2018tune}
Richard Liaw, Eric Liang, Robert Nishihara, Philipp Moritz, Joseph~E Gonzalez,
  and Ion Stoica.
\newblock Tune: A research platform for distributed model selection and
  training.
\newblock {\em arXiv preprint arXiv:1807.05118}, 2018.

\bibitem{longuet-higginsComputerAlgorithmReconstructing1981}
H.~C. {Longuet-Higgins}.
\newblock A computer algorithm for reconstructing a scene from two projections.
\newblock {\em Nature}, 293(5828):133--135, Sept. 1981.

\bibitem{loweDistinctiveImageFeatures2004}
David~G. Lowe.
\newblock Distinctive {{Image Features}} from {{Scale}}-{{Invariant
  Keypoints}}.
\newblock {\em International Journal of Computer Vision}, 60(2):91--110, Nov.
  2004.

\bibitem{maron2018invariant}
Haggai Maron, Heli Ben-Hamu, Nadav Shamir, and Yaron Lipman.
\newblock Invariant and equivariant graph networks.
\newblock In {\em ICLR}, 2018.

\bibitem{maron2019universality}
Haggai Maron, Ethan Fetaya, Nimrod Segol, and Yaron Lipman.
\newblock On the universality of invariant networks.
\newblock In {\em Int. Conf. Machine Learning}, pages 4363--4371, 2019.

\bibitem{Melnyk_2021_ICCV}
Pavlo Melnyk, Michael Felsberg, and M\r{a}rten Wadenb\"ack.
\newblock Embed me if you can: A geometric perceptron.
\newblock In {\em Proceedings of the IEEE/CVF International Conference on
  Computer Vision (ICCV)}, pages 1276--1284, October 2021.

\bibitem{pytorch_NEURIPS2019_9015}
Adam Paszke, Sam Gross, Francisco Massa, Adam Lerer, James Bradbury, Gregory
  Chanan, Trevor Killeen, Zeming Lin, Natalia Gimelshein, Luca Antiga, Alban
  Desmaison, Andreas Kopf, Edward Yang, Zachary DeVito, Martin Raison, Alykhan
  Tejani, Sasank Chilamkurthy, Benoit Steiner, Lu Fang, Junjie Bai, and Soumith
  Chintala.
\newblock Pytorch: An imperative style, high-performance deep learning library.
\newblock In H. Wallach, H. Larochelle, A. Beygelzimer, F. d\textquotesingle
  Alch\'{e}-Buc, E. Fox, and R. Garnett, editors, {\em Advances in Neural
  Information Processing Systems 32}, pages 8024--8035. Curran Associates,
  Inc., 2019.

\bibitem{qi2017pointnet}
Charles~R Qi, Hao Su, Kaichun Mo, and Leonidas~J Guibas.
\newblock {PointNet}: Deep learning on point sets for {3D} classification and
  segmentation.
\newblock In {\em CVPR}, pages 652--660, 2017.

\bibitem{rapp2020lidar}
Joshua Rapp, Julian Tachella, Yoann Altmann, Stephen McLaughlin, and Vivek~K
  Goyal.
\newblock Advances in single-photon lidar for autonomous vehicles: Working
  principles, challenges, and recent advances.
\newblock {\em IEEE Signal Processing Magazine}, 37(4):62--71, 2020.

\bibitem{rudin}
Walter Rudin.
\newblock {\em Principles of Mathematical Analysis}.
\newblock McGraw-Hill, 1953.

\bibitem{sun-cvpr-2020}
Weiwei Sun, Wei Jiang, Eduard Trulls, Andrea Tagliasacchi, and Kwang~Moo Yi.
\newblock {ACNe}: Attentive context normalization for robust
  permutation-equivariant learning.
\newblock In {\em CVPR}, 2020.

\bibitem{thomasTensorFieldNetworks2018}
Nathaniel Thomas, Tess Smidt, Steven Kearnes, Lusann Yang, Li Li, Kai Kohlhoff,
  and Patrick Riley.
\newblock Tensor field networks: Rotation- and translation-equivariant neural
  networks for {{3D}} point clouds.
\newblock {\em arXiv:1802.08219 [cs]}, May 2018.

\bibitem{thomeeYFCC100MNewData2016}
Bart Thomee, David~A. Shamma, Gerald Friedland, Benjamin Elizalde, Karl Ni,
  Douglas Poland, Damian Borth, and Li-Jia Li.
\newblock {{YFCC100M}}: The new data in multimedia research.
\newblock {\em Commun. ACM}, 59(2):64--73, Jan. 2016.

\bibitem{vaswaniAttentionAllYou2017}
Ashish Vaswani, Noam Shazeer, Niki Parmar, Jakob Uszkoreit, Llion Jones,
  Aidan~N Gomez, {\L}ukasz Kaiser, and Illia Polosukhin.
\newblock Attention is {{All}} you {{Need}}.
\newblock In {\em Advances in {{Neural Information Processing Systems}}},
  volume~30. {Curran Associates, Inc.}, 2017.

\bibitem{villar2021scalars}
Soledad Villar, David~W. Hogg, Kate Storey-Fisher, Weichi Yao, and Ben
  Blum-Smith.
\newblock Scalars are universal: Equivariant machine learning, structured like
  classical physics.
\newblock {\em Preprint. arXiv: 2106.06610}, 2021.

\bibitem{pmlr-v97-wagstaff19a}
Edward Wagstaff, Fabian Fuchs, Martin Engelcke, Ingmar Posner, and Michael~A.
  Osborne.
\newblock On the limitations of representing functions on sets.
\newblock In Kamalika Chaudhuri and Ruslan Salakhutdinov, editors, {\em
  Proceedings of the 36th International Conference on Machine Learning},
  volume~97 of {\em Proceedings of Machine Learning Research}, pages
  6487--6494. PMLR, 09--15 Jun 2019.

\bibitem{weiler_cesa_2019}
Maurice Weiler and Gabriele Cesa.
\newblock General e(2)-equivariant steerable cnns.
\newblock In H. Wallach, H. Larochelle, A. Beygelzimer, F. d\textquotesingle
  Alch\'{e}-Buc, E. Fox, and R. Garnett, editors, {\em Advances in Neural
  Information Processing Systems}, volume~32. Curran Associates, Inc., 2019.

\bibitem{weilerCoordinateIndependentConvolutional2021}
Maurice Weiler, Patrick Forr{\'e}, Erik Verlinde, and Max Welling.
\newblock Coordinate {{Independent Convolutional Networks}} -- {{Isometry}} and
  {{Gauge Equivariant Convolutions}} on {{Riemannian Manifolds}}.
\newblock {\em arXiv:2106.06020 [cs, stat]}, June 2021.

\bibitem{weiler2018cvpr}
Maurice Weiler, Fred~A. Hamprecht, and Martin Storath.
\newblock Learning steerable filters for rotation equivariant {CNNs}.
\newblock In {\em CVPR}, 2018.

\bibitem{worrall2017cvpr}
Daniel~E. Worrall, Stephan~J. Garbin, Daniyar Turmukhambetov, and Gabriel~J.
  Brostow.
\newblock Harmonic networks: Deep translation and rotation equivariance.
\newblock In {\em CVPR}, 2017.

\bibitem{xieAttentionalShapeContextNetPoint2018}
Saining Xie, Sainan Liu, Zeyu Chen, and Zhuowen Tu.
\newblock Attentional {{ShapeContextNet}} for {{Point Cloud Recognition}}.
\newblock In {\em 2018 {{IEEE}}/{{CVF Conference}} on {{Computer Vision}} and
  {{Pattern Recognition}}}, pages 4606--4615, June 2018.

\bibitem{Xu_2021_ICCV}
Jianyun Xu, Xin Tang, Yushi Zhu, Jie Sun, and Shiliang Pu.
\newblock Sgmnet: Learning rotation-invariant point cloud representations via
  sorted gram matrix.
\newblock In {\em Proceedings of the IEEE/CVF International Conference on
  Computer Vision (ICCV)}, pages 10468--10477, October 2021.

\bibitem{xuShowAttendTell2015}
Kelvin Xu, Jimmy Ba, Ryan Kiros, Kyunghyun Cho, Aaron Courville, Ruslan
  Salakhudinov, Rich Zemel, and Yoshua Bengio.
\newblock Show, {{Attend}} and {{Tell}}: Neural {{Image Caption Generation}}
  with {{Visual Attention}}.
\newblock In {\em Proceedings of the 32nd {{International Conference}} on
  {{Machine Learning}}}, pages 2048--2057. {PMLR}, June 2015.

\bibitem{yaoSimpleEquivariantMachine2021}
Weichi Yao, Kate {Storey-Fisher}, David~W. Hogg, and Soledad Villar.
\newblock A simple equivariant machine learning method for dynamics based on
  scalars.
\newblock {\em arXiv:2110.03761 [cs]}, Oct. 2021.

\bibitem{yarotsky2021universal}
Dmitry Yarotsky.
\newblock Universal approximations of invariant maps by neural networks.
\newblock {\em Constructive Approximation}, pages 1--68, 2021.

\bibitem{moo-cvpr-2018}
Kwang~Moo Yi, Eduard Trulls, Yuki Ono, Vincent Lepetit, Mathieu Salzmann, and
  Pascal Fua.
\newblock Learning to find good correspondences.
\newblock In {\em CVPR}, 2018.

\bibitem{zaheer2018deep}
Manzil Zaheer, Satwik Kottur, Siamak Ravanbakhsh, Barnabas Poczos, Russ~R
  Salakhutdinov, and Alexander~J Smola.
\newblock Deep sets.
\newblock In {\em NeurIPS}, volume~30, 2017.

\bibitem{zhang2019oanet}
Jiahui Zhang, Dawei Sun, Zixin Luo, Anbang Yao, Lei Zhou, Tianwei Shen, Yurong
  Chen, Long Quan, and Hongen Liao.
\newblock Learning two-view correspondences and geometry using order-aware
  network.
\newblock {\em International Conference on Computer Vision (ICCV)}, 2019.

\bibitem{Zhao_2021_ICCV}
Hengshuang Zhao, Li Jiang, Jiaya Jia, Philip~H.S. Torr, and Vladlen Koltun.
\newblock Point transformer.
\newblock In {\em ICCV}, pages 16259--16268, October 2021.

\bibitem{Zhong_2021_ICCV}
Zhen Zhong, Guobao Xiao, Linxin Zheng, Yan Lu, and Jiayi Ma.
\newblock {T-Net}: Effective permutation-equivariant network for two-view
  correspondence learning.
\newblock In {\em ICCV}, 2021.

\end{thebibliography}
}


\newpage
\onecolumn

\appendix
\section{Proofs}

Here, we provide proofs, and other theoretical details, left out in the the main text.

\subsection{Spaces of point clouds}
  In the main paper, we have, in the interest of readability, intentionally refrained from being too formal. In particular, we have equated point clouds with vectors in $\C^m$ in a quite streamlined fashion. As we want to present formal proofs here, this will no longer suffice. In particular, since we are aiming to apply the Stone-Weierstrass theorem,  we will need to consider the point clouds as points in a metric space. We will therefore consider the following well-known approach (the same ideas were applied in e.g. \cite{keriven2019universal,maron2019universality}.)
\begin{defi} \label{defi:metric_spaces}For a subgroup $G\sse S_m$, let $\sim_G$ denote the equivalence relation $$Z\sim_G W \Leftrightarrow \ \exists \pi \in G, Z = \pi^*W$$ on $\C^m$. We can equip the set of equivalence classes $\C^m / \sim_G$ with the metric $$d_G(Z,W)~=~\inf_{\pi \in G} \norm{Z-\pi^*W}.$$ 
For $G=S_m$, we denote the resulting metric space $\calP^m$. For $G=\Stab(0)$, we denote it $\calP_0^m$.

On $\calP^m$ and $\calP^m_0$, we may define a further equivalence relation via $Z\sim_{\mathbb{S}} W \Leftrightarrow Z=\theta W$ for some $\theta \in \mathbb{S}$. We can again define a metric on the set of equivalence classes under this relation via 
   $$ d_{\mathbb{S}}(Z,W) = \inf_{\theta \in \mathbb{S}} d_G(Z,\theta W),$$
where $d_G$ is the metric from above. We call the resulting metric spaces $\calR\calP^m$ and $\calR\calP_0^m$
\end{defi}

In the following, we will without comment equip all spaces of continuous functions with the topology induced by the supremum norm on compact sets. If $M$ is a metric space, we let $\calC(M)$ denote the space of complex-valued continuous functions on $M$.

\begin{rem}
    (i) It is clear that permutation invariant functions $F\in \calC(\C^m)$ can be identified with functions in $\calC(\calP^m)$. If they are additionally rotation invariant, we can even identify them with functions on $\calC(\calRP^m)$. Similar statements hold for $\calC(\calP_0^m)$ and $\calC(\calRP_0^m)$.
    
    (ii) In the following, we will sometimes consider expressions in which functions defined on $\calP_0^m$, or $\C^m$, are applied to members in $Z \in \calP^m$. This is clearly in general not formally well-defined. However, in each such expression, there are other operations present which makes the object per se well defined again. For instance, $\nu_i(Z) = \abs{z_i}$ is not well defined on $\calP^m$, but $\nu(Z) = \sup_{i \in [m]} \abs{z_i}$ is. In the interest of readability, we will not comment on this in detail every time.
\end{rem}

\subsection{Proof of Proposition \ref{prop:nogo1}} \label{sec:nogoPN}

Let us begin by proving the no-go result of Proposition \ref{prop:nogo1}, stating that the most straighforward way of making the pointnet architecture rotation equivariant will not yield a universal architecture.

\begin{proof}[Proof of Proposition \ref{prop:nogo1}]
    Let us call a cloud $Z$ for which all points have the same norm and obey $\sum_{i \in [m]}z_i =0$ \emph{balanced}. We claim that every function of the form $\chi(\sum_{i \in [m]} \varrho(z_i))$ is constant on the set of balanced clouds.
    
    To see this, let us first notice that if $\varrho: \C\to \C^K$ is rotation equivariant, it must be possible to write it on the form $\varrho(z) = \nu(\abs{z})z$ for some function $\nu: \R_{+} \to \C^K$. A formal way to prove this is to notice that the function $z \mapsto \overline{z}\varrho(z)$ is rotation equivariant, and hence can only depend on the modulus of $z$.

Now, if $r$ is the common value for the norms in a balanced cloud $Z$, we have 
    \begin{align*}
        \chi \big(\sum_{i\in[m]} \varrho(z_i) \big) =\chi \big( \nu(r) \sum_{i \in [m]} z_i\big) =\chi(0).
    \end{align*}
   Hence $\chi(Z)=\chi(0)$ for all such clouds. To finish the proof, it is therefore enough to prove the existence function $f \in \calR(m)$ that is not  constant on the set of balanced clouds.

    Towards this endeavour, let $a:\R\to \C$ be a function and consider 
    \begin{align*}
        f(Z) = \sum_{i< j \in [m]} a\left( \abs{z_i-z_j}\right) \cdot \sum_{k\notin \set{i,j}}z_k
    \end{align*}
   That is, in words; First, for each pair $z_i, z_j$ of points, calculate $a(\abs{z_i-z_j})$ and multiply that with the sum of the rest of the points. Then sum over the set of such pairs. It is not hard to realize that such functions are members of $\calR(m)$.
    
    Now let $a$, for some $\epsilon>0$, be equal to $1$ in $\sfrac{2\sqrt{5}}{3}$ and zero outside $[\sfrac{2\sqrt{5}}{3}-\epsilon,\sfrac{2\sqrt{5}}{3}+\epsilon]$. Then, if $Z$ is a cloud with all pairwise distances smaller than $\sfrac{2\sqrt{5}}{3}-\epsilon$, $f(Z)=0$. There exists balanced clouds with that property for all $m$. Therefore, if $f$ is constant on the set of balanced clouds, we must have $f(Z)=0$ for all such. We can however construct a balanced cloud for which $f(Z)\neq 0$ as follows:

    Let us first assume that $m=2k+3$ is odd. We define
    \begin{align*}
        z_1=i, \quad z_{2,3}= -\frac{2i}{3}\pm \frac{\sqrt{5}}{3}, \quad z_{2\ell, 2\ell+1} = \frac{i}{6k} \pm \frac{\sqrt{36k^2-1}}{6k}.
    \end{align*}
    All points in these clouds have the norm $1$, and
    \begin{align*}
        \sum_{k=1}^{m} z_k = i  -\frac{4i}{3} + 2k \cdot  \frac{i}{6k} = 0.
    \end{align*}
    Note that we used that the real parts of the points cancel each other. Thus, the set is balanced. (See also Fig. \ref{fig:balanced}).
    
    \begin{figure}
        \centering
        \includegraphics[width=.3\textwidth]{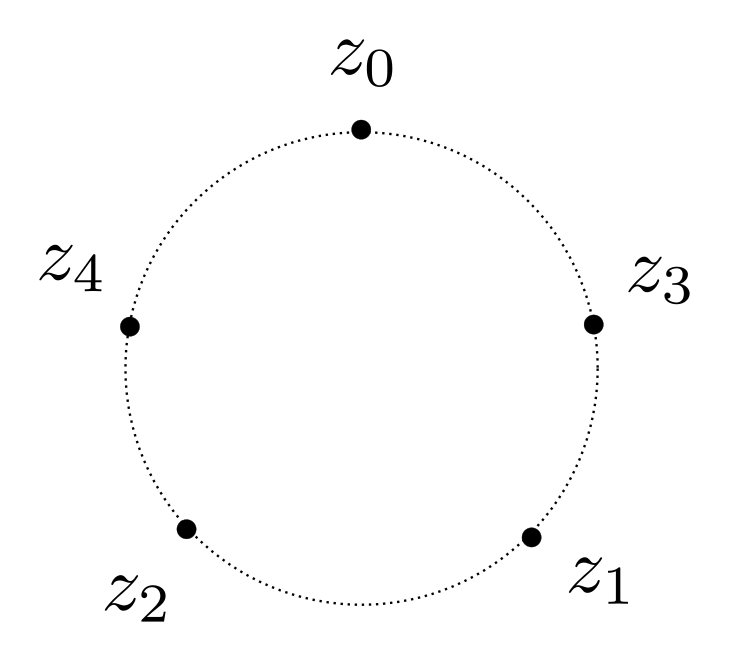}
        \caption{The balanced cloud $Z$ used in the proof of Prop \ref{prop:nogo1}. Note that among the pairwise distances $\abs{z_i-z_j}$, only $\abs{z_1-z_2}$ is equal to $\sfrac{2\sqrt{5}}{3}$.}
        \label{fig:balanced}
    \end{figure}
    
    Now, by calculating all distances between points, we see that  $\abs{z_1-z_2}=\sfrac{2\sqrt{5}}{3}$, and that all other pairwise distances $\abs{z_i-z_j}$ for $(i,j)\neq(1,2)$ are unequal to  $\sfrac{2\sqrt{5}}{3}$. Therefore, if we choose the parameter $\epsilon$ from above small enough, we get 
    \begin{align*}
        a(\abs{z_i-z_j}) = \begin{cases} 1 & \text{ if } i=1,j=2 \\
        0 & \text{ else,}
        \end{cases}
    \end{align*}
    and 
    \begin{align*}
        f(Z) &= \sum_{k \notin \{1,2\}} z_k = i + 2k  \cdot  \frac{i}{6k} = \frac{4i}{3}\neq 0.
    \end{align*}
Hence, $f$ is not constant on the set of balanced clouds, and the argument is finished.

In the case of even $m$, we proceed as above, but interchange $z_0=i$ with the two points $z_{-1,0} = 0.5i\pm \frac{\sqrt{3}}{2}$. The argument then proceeds just as above.
\end{proof}

\subsection{Proof of Theorem \ref{th:denseset}} \label{sec:dense}

Here, we prove that functions of the form $\sum_{i \in [m]} \gamma(\tau_i^*Z)z_i$ are dense in $\calR(m)$. Before starting the actual proof, let us agree on a simplifying notational convention.  For a  complex polynomial $q$, we will refer to the function
\begin{align*}
    p(Z) = q(Z,\overline{Z}).
\end{align*}
as a \emph{real polynomial} in $Z$. Note that the set of these functions is dense in $\calC(\C^m)$ with respect to supremum norm \emph{compact sets}, n.b.. To see this, note that the classical Stone-Weierstrass theorem states that for any $N\in \N$ the set of real polynomials is dense in $\calC(\R^n)$.  By equating $\C^m$ with $\R^{2m}$, we see that the space of real polynomials \emph{in the real and imaginary parts} of $Z \in \C^m$,
\begin{align*}
    r(\mathrm{re}(Z),\mathrm{im}(Z))
\end{align*}is dense  $\calC(\C^m)$. Since we however for each such polynomial $r$ can find a complex $q$ with $r(\mathrm{re}(Z),\mathrm{im}(Z))=q(Z, \overline{Z})$ for all $Z$, the claim follows.

Having established that density result we now move on to prove that in order to approximate functions in $\calR(m)$, it is enough to consider polynomials with the same equivariance properties. Similar statements have been proven in e.g. \cite{yarotsky2021universal,dym-maron-2021}. 

    \begin{lemma} \label{lem:equipolys}
    The set of real polynomials $p$ that are permutation invariant and rotation equivariant is dense in $\calR(m)$.
\end{lemma}
\begin{proof}

Let us first prove that it suffices to consider rotation equivariant polynomials, we argue as follows. For some multi-indices $\alpha, \beta \in \N^m$, consider the 'real monomial'
\begin{align*}
    \mu_{\alpha \beta} (Z) =  Z^{\alpha} \overline{Z}^\beta.
\end{align*}
It is clear that $\mu_{\alpha\beta}$ is rotationally equivariant if and only if $\abs{\alpha}=\abs{\beta}+1$. This together with the fact that $ \tfrac{1}{2\pi} \int_{\sph} \theta^k \dd \theta = \delta_{k,0}$ implies that
   \begin{align} \label{eq:roteq}
       \int_{\sph} \overline{\theta} \mu_{ \alpha \beta}(\theta Z) \, \dd \theta \neq 0 \quad \Longleftrightarrow \quad \mu_{ \alpha \beta}  \text{ rotationally equivariant.}
   \end{align}
   Also notice that if $f$ is rotationally equivariant,
   \begin{align*}
        \tfrac{1}{2\pi}\int_\S \overline{\theta}f(\theta Z) \,  \dd \theta = \tfrac{1}{2\pi}\int_\S f(Z) \,  \dd  \theta= f(Z).
   \end{align*}
   
   Now fix a compact set $K \sse \C^m$, which without loss of generality has the property $Z \in K \Leftrightarrow \theta Z \in K$, $\theta \in \S$. For every $f \in \calR(m)$, there exists a real polynomial $p$ with $\sup_{Z \in K}\abs{p(Z)-f(Z)}\leq \epsilon$. We now split the monomial terms in $p$ according to whether they are rotationally equivariant or not. This defines two polynomials $p_0$ and $p_1$. Now notice that for each $Z \in K$,
   \begin{align*}
       \abs{f(Z)-p_0(Z)} = \left\vert\tfrac{1}{2\pi} \int_{\sph} \overline{\theta} (f(\theta Z) -  p_0(\theta Z)  - p_1 (\theta Z)) \, \dd \theta \right\vert \leq \sup_{Z \in K} \abs{f(Z)-p(Z)}
   \end{align*}
   We used that $p_0$ and $f$ are rotationally equivariant, and also  \eqref{eq:roteq} together with the fact that $p_1$ only consists of monomial terms that are not rotationally equviarant.
   This means that the rotationally equivariant real polynomial $p_0\in \calR(m)$ has a supremum distance at most $\epsilon$ to $f$ on $K$, and we hence we might as well use $q$ to approximate $f$.
   
   The permutation invariance part is now easily handeled by symmetrization. That is if, $p$ is a non-symmetric polynomial approximating $f$ well, the symmetric polynomial
   \begin{align*}
       \widehat{p}(Z) = \tfrac{1}{\abs{S_m}} \sum_{\pi \in S_m}p(\pi^*Z)
   \end{align*}
   will approximate $f$ just as good --  see for instance  \cite{maron2019universality}.
\end{proof}

With the previous lemma in our toolbox, the proof of Theorem \ref{th:denseset} is relatively simple.

\begin{proof}[Proof of Theorem \eqref{th:denseset}]
    Fix a compact set and a function $f$. By Lemma \ref{lem:equipolys}, there exists a real, symmetric and rotation equivariant polynomial
    \begin{align*}
        p(Z) = \sum_{\alpha, \beta} c_{\alpha,\beta} Z^\alpha \overline{Z}^\beta,
    \end{align*}
   which is close to $f$. Since  $p$ is rotation invariant, it must be  $c_{\alpha,\beta} = 0$ for all ($\alpha$,$\beta$) with $\abs{\alpha}\neq\abs{\beta}+1$. Due to its permutation invariance,  we furthermore have $c_{\alpha,\beta}=c_{\pi^*\alpha,\pi^*\beta}$ for all $\pi \in S_m$ and multiindices $\alpha,\beta$. Hence, $p$ consists of terms of the form
   \begin{align}
        q_{\alpha,\beta}(Z)= \sum_{\pi \in S_m}Z^{\pi^*\alpha} \overline{Z}^{\pi^*\beta}, \ \abs{\alpha}=\abs{\beta}+1. \label{eq:basispolynomial},
    \end{align}
    and it is therefore enough to approximate such terms. 
    Here, by the permutation equivariance, we can WLOG assume that the indices $\alpha_i$ are in ascending order. Consequently, we can write $\alpha = \hat{\alpha}+e_0$ for some $\hat{\alpha}$ with $\abs{\hat\alpha}= \abs{\beta}$.
    
    Now let us split the sum in \eqref{eq:basispolynomial} over $S_m$ in accordance to the value of $\pi(0)$
    \begin{align} \label{eq:split}
         q_{\alpha,\beta}(Z)  = \sum_{i \in [m]} \sum_{\pi(0)=i}Z^{\pi^*(e_0 + \hat{\alpha})} \overline{Z}^{\pi^*\beta} = \sum_{i \in [m]} \sum_{\pi(i)=0}Z^{\pi^* \hat{\alpha}} \overline{Z}^{\pi^*\beta} z_i,
    \end{align}
    where we in the last step used that $\pi^*e_0=e_{\pi(0)}=e_i$. It is clear that we can write each $\pi$ with $\pi(0)=i$ as $\tau_i \circ \sigma$ for a unique $\sigma \in \Stab(0)$. We have
    \begin{align*}
        Z^{\pi^*\hat{\alpha}} \overline{Z}^{\pi^*\beta} &= Z^{\tau_i^*\sigma^*\hat{\alpha}} \overline{Z}^{\tau_i^*\sigma^*\hat{\beta}}z_i 
    \end{align*}
    Since  $(\tau_i^* Z)^\alpha=Z^{\tau_i^*\alpha}$, we see that our sum turns into 
    \begin{align*}
       \sum_{i \in [m]} \sum_{\sigma \in \Stab(0)} (\tau_i^*Z)^{\sigma^* \hat{\alpha}} (\tau_i^* \overline{Z})^{\sigma^* \beta} z_i = \sum_{i \in [m]} \gamma(\tau_i^*Z)z_i,
    \end{align*}
    where we defined
    \begin{align*}
        \gamma(Z) = \sum_{\sigma \in \Stab(0)} Z^{\sigma^* \hat{\alpha}}  \overline{Z}^{\sigma^* \beta}
    \end{align*}
 The function $\gamma$ is clearly $\Stab(0)$-invariant, and also rotation invariant due to $\abs{\hat{\alpha}}=\abs{\beta}$. The proof is finished.
\end{proof}

\subsection{\texorpdfstring{$\Stab(0)$}{Stab(0)}-equi- and invariant linear maps} \label{sec:linearlayers}

Our architectures make heavy use of linear layers which are equi- and invariant to the action of the $\Stab(0)$ group. It is a priori not clear how to construct such, and in particular parametrize all of them. In this section, we provide such a description.

We let $\K$ denote either of the fields $\R$ or $\C$. For a tensor $T \in (\K^m)^{\otimes k}$, i.e. of order $k$, we define the action of a permutation $\pi \in S_m$ on $T$ through
\begin{align*}
    (\pi^*T)_{i_0, \dots, i_{k-1}} = T_{\pi^{-1}(i_0), \dots, \pi^{-1}(i_{k-1})}.
\end{align*}
This is exactly as in \cite{maron2018invariant}. Let us begin by introducing some notation for the spaces we are interested in.
\begin{defi}
    For $k, \ell \in \N$, we let $\calL(k,\ell)$ denote the space of linear operators $L: (\C^m)^{\otimes k} \to (\C^m)^{\otimes \ell}$ which are $S_m$-equivariant. The space of operators of the same kind which are $\Stab(0)$-equivariant is denoted $\calL_0(k,\ell)$.
\end{defi}

Let us briefly comment on two special cases. First, if $\ell=0$, the spaces $\calL(k,0)$ and $\calL_0(k,0)$ can be identified with the space of invariant functionals of the respective kind. This is because of the fact that the action of $S_m$ on scalars $v\in \K$ is trivial. In the same manner, the spaces $\calL(0,k)$ and $\calL_0(0,k)$ denote constant $k$-tensors which are invariant to the action of the respective groups. Such elements can be used as biases in our architecture.

\begin{rem}
    In our architecture, we are actually dealing with linear layers mapping multi-tensors to multi-tensors. It is however clear that such a mapping can be seen as a matrix of linear maps $L_{ij}$, where each $L_{ij}$ corresponds to one input-output-channel pair. As such, it is enough to characterize the spaces $\calL(k,\ell)$ and $\calL_0(k,\ell)$ to obtain a way to parametrize the linear layers of our architecture.
\end{rem}

Let us reiterate that the results of  \cite{maron2018invariant} give a complete characterization of the spaces $\calL(k,\ell)$\footnote{Technically, they only state their theorems in the case $\K=\R$, but their proofs go through also for $\K=\C$}. In brief, they identify such maps as fixed points of a certain linear equation, which they then explicitly calculate.  We refer to \cite{maron2018invariant} for details.

In particular, the results in the mentioned paper prove that $\dim \calL(k,\ell) \leq B_{k+\ell}$, where $B_n$ denotes the $n$:th \emph{Bell number}. As noted in \cite{finziPracticalMethodConstructing2021}, the dimension of the space cannot get larger than the dimension of the space of all linear maps from $(\K^m)^{\otimes k}$ to $(\K^m)^{\otimes \ell}$, which is $m^{k+\ell}$. In all cases, the number of scalars needed to describe a map in $\calL(k,\ell)$ can be bounded independent of $m$. 

Our idea here is to link the spaces $\calL_0(k,\ell)$ with spaces $\calL(k',\ell')$. In doing so, the following simple Lemma will be convenient . For completeness. we include a proof.
\begin{lemma} \label{lem:duality} 
    For $k,\ell$ in $N$, consider the map
    \begin{align*}
        \Phi_{k,\ell}: L \mapsto \lambda, \quad \lambda(\overline{S}\otimes T) = \sprod{S,L(T)}, \ T\in (\K^m)^{\otimes k} , S\in (\K^m){\otimes \ell}.
    \end{align*}
    Hereby, $\sprod{\cdot, \cdot}$ denotes the canonical scalar product on $(\K^m)^{\otimes \ell}$, i.e.
    \begin{align*}
        \sprod{M,N} = \sum_{i_0, \dots, i_{\ell-1}} \overline{M_{i_0, \dots, i_{\ell-1}}} N_{i_0, \dots, i_{\ell-1}}
    \end{align*}
     \begin{enumerate}[(i)]
         \item $\Phi_{k,\ell}$ is an isomorphism between the spaces of linear maps $(\K^m)^{\otimes k} \to (\K^m)^{\otimes \ell}$ and functionals on $(\K^m)^{\otimes(k+\ell)}$.
         \item $\Phi_{k\ell}$ maps  $\calL(k,\ell)$ to $\calL(k+\ell,0)$ and $\calL_0(k,\ell)$ to $\calL_0(k+\ell,0)$. In particular, the respective pairs of spaces are isomorphic.
     \end{enumerate}
\end{lemma}

\begin{proof}
    To not overload the notation, we fix $k$ and $\ell$ and drop the index on $\Phi$.
    
    \underline{Ad (i):} The linearity is evident. For proving the injectivity, suppose that $\lambda=\Phi(L)$ is the zero functional. That means per definition that $\sprod{S,L(T)}=0$ for all $S\in (\K^m)^{\otimes \ell}$, which implies that $L(T)=0$ for all $T$ in $(\K^m)^{\otimes k}$, i.e. that $L=0$. The surjectivity now follows from dimensionality considerations.
    
    \underline{Ad (ii)} We concentrate on the case of $S_m$-equivariant maps, since the $\Stab(0)$-case is proven in exactly the same way. We need to prove two things: First, we need to show that $\Phi(L)\in \calL(k+\ell,0)$ for all $L \in \calL(k,\ell)$. Secondly, we need to show that for every $\lambda \in \calL(k+\ell,0)$, the (unique) $L$ with $\Phi(L)=\lambda$ is in $\calL(k,\ell)$.
    
    To prove the former, let $L \in \calL(k,\ell)$ and $\pi \in S_m$ be arbitrary. Writing $\lambda = \Phi(L)$, we have
    \begin{align*}
        \lambda(\pi^*(\overline{S} \otimes T)) = \sprod{\pi^*S, L(\pi^*T)} = \sprod{\pi^*S, \pi^*L(T)}  = \sprod{S, L(T)} = \lambda(\overline{S} \otimes T),
    \end{align*}
   for each $S$ and $T$. Note that we used the equivariance of $L$ in the second step, and the (obvious) invariance of the scalar product under permutations in the third. This exactly means that $\lambda \in \calL(k+\ell,0)$.
   
   To prove the latter, let $\lambda \in \calL(k+\ell,0)$, $L=\Phi^{-1}(\lambda)$ and $\pi \in S_m$ arbitrary. For $S$ and $T$ arbitrary, defining $R= (\pi^{-1})^*S$, we then get
   \begin{align*}
       \sprod{S, L(\pi^*T)} =  \sprod{\pi^*R, L(\pi^*T)} = \lambda(\pi^*(\overline{R}\otimes T)) = \lambda(\overline{R} \otimes T) = \sprod{R,L(T)} = \sprod{\pi^*R,\pi^*L(T)} = \sprod{S, \pi^*L(T)}.
   \end{align*}
   We used the invariance of $\lambda$ in the third step, and the invariance of the scalar product in the fifth. Since $S$ is arbitrary, this proves that $L(\pi^*T)= \pi^*L(T)$ for all $T$, i.e., $L \in \calL(k,\ell)$
\end{proof}

The above lemma links spaces of equivariant linear maps to spaces of invariant functionals, in an isomorphic fashion. This means that in order to link the spaces $\calL(k,\ell)$ to the spaces $\calL_0(k,\ell)$, it suffices to provide a link between one space of functionals of the one kind to a space of equivariant maps of the other. This is the purpose of the following theorem.

\begin{theo} \label{theo:stabvsperm}
The map
    \begin{align*}
      \Psi: \calL(k,1) \to \calL_0(k,0),  L \mapsto \lambda, \quad   \lambda(T )=\sprod{e_0,L(T)}
    \end{align*}
    is an isomorphism. In particular, $\calL_0(k,0)\simeq \calL(k,1)$ and $\dim(\calL_0(k,0)) = B_{k+1}$.
\end{theo}
\begin{proof}
 Let us begin by proving that $\Psi$ is well-defined, i.e. that $\Psi(L) \in \calL_0(k,0)$ for each $L \in \calL(k,1)$.  Let $\sigma \in \Stab(0)$ be arbitrary. Due to the equivariance of $L$ and invariance of the scalar product, we then get 
\begin{align*}
    \lambda(\sigma^*T) = \sprod{e_0, L(\sigma^*T)} = \sprod{e_0, \sigma^*L(T)} = \sprod{(\sigma^{-1})^* e_0, L(T)} = \sprod{e_0,L(T)} = \lambda(T).
\end{align*}
In the penultimate step, we used that $(\sigma^{-1})^*e_0=e_{\sigma^{-1}(0)} = e_0$ for $\sigma \in \Stab(0)$. This means that $\lambda$ is invariant, and that $\Psi$ indeed is well defined.

Now for the isomorphy. It is clear that $\Psi$ is linear. To prove injectivity, assume that $\lambda= \Psi(L)=0$. Due to the equivariance of $L$, we then get for every $i \in [m]$ and $T \in (\K^m)^{k}$
\begin{align*}
    0=\lambda(\tau_i^*T) = \sprod{e_0,L(\tau_i^*T)} = \sprod{\tau_i^*e_0, L(T) } = \sprod{e_i, L(T)}.
\end{align*}
i.e. $L=0$. To show surjectivity, let $\lambda \in \calL_0(k,0)$ be arbitrary. Define a map $L : (\K^m)^{\otimes k} \to \K$ through
\begin{align*}
    \sprod{e_i,L(T)} = \lambda(\tau_i^*T), \quad i \in [m]
\end{align*}
We then have $\Psi(L)(T) = \sprod{e_0,L(T)}=\lambda(\tau_0^*T)=\lambda(T)$, i.e., $\lambda = \Psi(L)$. If we can prove that $L$ is equivariant, we are done. So let $\pi\in S_m$ and $i \in [m]$ be arbitrary. A direct computation shows that $\tau_i \circ \pi \circ \tau_{\pi^{-1}(i)} \in \Stab(0)$. This, together with the assumed invariance of $\lambda$, shows that 
\begin{align*}
 \sprod{e_i,L(\pi^*T)} &=  \lambda(\tau_i^*\pi^*T) =  \lambda(\tau_i^*\pi^* \tau_{\pi^{-1}(i)}^* \tau_{\pi^{-1}(i)}^* T) = \lambda( \tau_{\pi^{-1}(i)}^* T)  = \sprod{e_{\pi^{-1}(i)},L(T)} \\
 & = \sprod{(\pi^{-1})^*e_i, L(T)} = \sprod{e_i, \pi^*L(T)}.
\end{align*}
Since $i$ is arbitrary, this means that $L(\pi^*T)=\pi^*L(T)$, i.e., that $L$ is equivariant. The proof is finished.
\end{proof}

We can now use Lemma \ref{lem:duality} and Theorem \ref{theo:stabvsperm} to construct an isomorphism between $\calL_0(k,\ell)$ and  $\calL(k,\ell+1)$

\begin{cor} \label{cor:isomorphism}
$\calL(k,\ell+1) \simeq \calL_0(k,\ell)$. An isomorphism is given by
\begin{align*}
    \Xi : \calL_0(k,\ell) \to \calL(k,\ell+1), L_0 \mapsto K, \quad K(T) = \sum_{i \in [m]}e_i \otimes \tau_i^*L_0(\tau_i^*T).
\end{align*}
\end{cor}
\begin{proof}
    If $\Phi_{k,\ell}$ and $\Psi$ are as in Lemma \ref{lem:duality} and Theorem \ref{theo:stabvsperm}, respectively, we define the isomorphism $\Xi$ through the following chain
      \begin{align*}
          \begin{matrix}
                \calL_0(k,\ell) \\ L_0 
          \end{matrix} \quad  \stackrel{\Phi_{k,\ell}}{\to} \quad  \begin{matrix}
                \calL_0(k+\ell,0) \\ \lambda_0 
          \end{matrix} \quad  \stackrel{\Psi^{-1}}{\to} \quad  \begin{matrix}
                \calL(k+\ell,1) \\ L 
          \end{matrix}  \quad  \stackrel{\Phi_{k+\ell,1}}{\to} \quad  \begin{matrix}
                \calL(k+\ell+1,0) \\ \lambda 
          \end{matrix} \quad  \stackrel{\Phi_{k,\ell+1}^{-1}}{\to} \quad  \begin{matrix}
                \calL(k,\ell+1) \\ K
          \end{matrix}.
      \end{align*}
      It now only is left to prove that $\Xi$ has the claimed form. For convenience, we named all of the intermediate objects above. For $u \in \K^m, S\in (\K^m)^{\otimes\ell}$ and $T \in (\K^m)^{\otimes k}$, we calculate
      \begin{align*}
          \sprod{u\otimes S, K(T)}= \lambda(\overline{u} \otimes \overline{S} \otimes T) = \sprod{u, L(\overline{S} \otimes T)} = \sum_{i \in [m]}\overline{u}_i \sprod{e_i, L(\overline{S}\otimes T)}.
      \end{align*}
      Now, notice that since $L \in \calL(k+\ell,1)$ and the scalar product is $S_m$-invariant, we have  
      \begin{align*}
          \sprod{e_i, L(\overline{S}\otimes T)} = \sprod{\tau_i^* e_0, L(\overline{S} \otimes T)} = \sprod{e_0, \tau_i^*L(\overline{S}\otimes T)} = \sprod{e_0, L(\tau_i^*(\overline{S}\otimes T))}.
      \end{align*}
      Consequently ,
      \begin{align*}
          \sum_{i \in [m]}\overline{u}_i \sprod{e_i, L(\overline{S}\otimes T)} &= \sum_{i \in [m]} \overline{u_i}\sprod{e_0, L(\tau_i^*(\overline{S}\otimes T))} = \sum_{i \in [m]}\overline{u_i} \lambda_0(\tau_i^*\overline{S}\otimes \tau_i^*T) = \sum_{i \in [m]}\overline{u_i} \sprod{\tau_i^*S, L_0(\tau_i^*T)} \\
          &= \sum_{i \in [m]}\overline{u_i} \sprod{S,  \tau_i^* L_0(\tau_i^*T)}  = \sum_{i \in [m]}\sprod{u,e_i} \sprod{S,  \tau_i^* L_0(\tau_i^*T)} = \sprod{u \otimes S, \sum_{i \in [m]} e_i \otimes \tau_i^*L_0(\tau_i^*T)}.
      \end{align*}
      Since $u$ and $S$ are arbitrary, we obtain the claim.
\end{proof}

We may now easily construct spanning systems of $\calL_0(k,\ell)$ by, using the last corollary, transforming the spanning sets of $\calL(k,\ell+1)$ from  \cite{maron2018invariant}. In Section \ref{sec:spanning sets}, we carry this out and write down explicit spanning sets for the spaces  $\calL_0(k,\ell)$ for $0 \leq k,\ell\leq 2$.

Let us here only comment that the above Corollary in particular proves that $\dim \calL_0(k,\ell)=\dim \calL(k,\ell+1) \leq B_{k+\ell+1}$. In particular, we may describe each linear input-output channel pair of the first layer of our weight units  (which is an element of $\calL_0(2,1)$)  with $B_{2+1+1}=15$ parameters, and each channel of the bias (which is an element of $\calL_0(0,1)$) with $B_{1+1}=2$ parameters. For the later early layers, we need $B_{1+1+1}=5$ parameters per input-output-channel pair for the linear part (which is then an element of $\calL_0(1,1)$) , and $B_{1+1}=2$ parameter per output channel bias (which is still an element of $\calL_0(0,1)$) . 

\subsection{Proof of Theorem \ref{th:universality}} \label{sec:universal}
We now prove the main result. Note that we have to assume that the activation function in the weight units is not a polynomial (in order to be able to apply the classical universality result for neural networks \cite{cybenko1989approximation}.) The first step is to prove that the  $\calNS(m)$-architecture, i.e. the ones for the weight units is universal for functions restricted to a subset of  $\calR\calP_0^m$.
\begin{lemma} \label{lem:weightsdense}
    For $\epsilon>0$, define the set
    \begin{align*}
        C_\epsilon^m &= \set{Z \in \calR\calP_0^m  \, \vert \, \abs{z_0} >\epsilon}.
    \end{align*}
    Then, $\mathcal{NS}(m)$ is dense in $\calC(C_\epsilon^m)$.
\end{lemma}

\begin{figure}
    \centering
    \includegraphics[width=.35\textwidth]{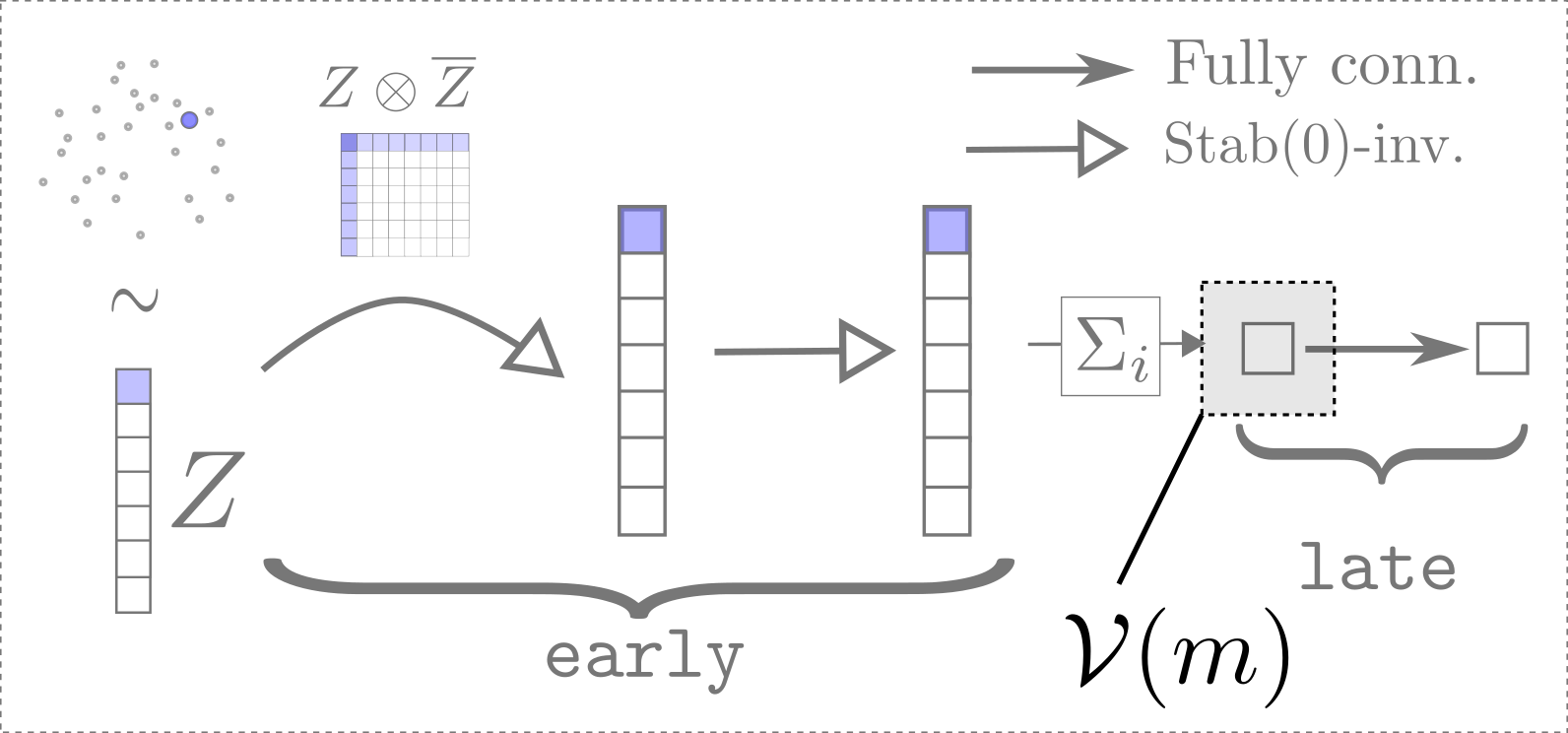}
    \caption{Definition of the space $\calV(m)$. \label{fig:Vm}}
\end{figure}

\begin{proof}
We aim to apply the the Stone-Weierstrass theorem \cite[Th.7.32]{rudin}. This theorem says that if a set $S$ of continuous functions defined on a compact metric space $M$
\begin{itemize}
    \item separates points, e.g. if there for each $x\neq y \in M$ exists an $f\in S$ with $f(x)\neq f(y)$,
    \item vanishes nowhere, e.g. that there for $x\in M$ exists a $f\in S$ such that $f(x)\neq 0$,
\end{itemize}
the algebra generated by $S$ is dense in $\calC(M)$. Note that we may use the real version of the theorem, since we are applying real-linear layers.

In our setting, we want to apply the theorem with $M$ equal to an arbitrary compact subset of  $C_\epsilon^m$, and $S$ equal to the functions $v$ defined by the averaging the output of the last early layer of the networks in $\calNS(m)$. For convenience, let us call this set $\calV(m)$. (See also Figure \ref{fig:Vm}.)  Due to the classical universality result of neural networks \cite{cybenko1989approximation}, the final fully connected layers can namely generate the algebra of those functions. 

That $\calV(m)$ is  nowhere vanishing is imminent, simply due to the fact that the linear layers have biases. Thus, we can concentrate on proving that it separates points. So let $Z\neq W \in C_\epsilon^m$. We aim to show that if $v(Z)=v(W)$ for all functions in $\calV(m)$, $Z$ must be equal to $W$ as points in $\calR\calP_0^m$, i.e., up to a $\Stab(0)$-permutation and global rotation. For convenience, let us introduce the notations $Z_\wedge = (0,z_1, \dots, z_{m-1})$ and $Z_\vee =(z_1, \dots, z_{m-1}) $

\paragraph{Claim 1: $\abs{z_0}=\abs{w_0}$.} The map $T \mapsto e_0 T_{00}$ is a member of $\calL_0(2,1)$. This can be seen through a direct calculation (see also Section \ref{sec:spanning sets}.) Therefore, channels of the  first layer of $\alpha$ can be chosen to output $\abs{z_0}^2e_0$. By choosing the subsequent input-output-channel pairs as multiples of the identity, it can therefore be achieved that the output of the $L$:th layer can be made equal to  $\psi(\abs{z_0}^{2})e_0$ for some neural network, which surely can be designed to be arbitrarily close to the identity (we hereby again appeal to the classical universality result). This vector is of course summed to (something arbitrarily close to) $\abs{z_0}^{2}$. Hence,  $\abs{z_0}^2$ can be approximated arbitrarily well with functions in $\calV(m)$, and consequently, $\abs{z_0}=\abs{w_0}$.

\paragraph{Claim 2: $z_0\overline{Z_\vee}=w_0 \overline{W_\vee}$ up to a permutation.} Now we use that the maps $T \mapsto Te_0$ and $T \mapsto T^Te_0$ are members of $\calL_0(2,1)$ (This can again be realized through a direct calculation, or a consultation of Section \ref{sec:spanning sets}). Since we apply such functions on $Z\otimes \overline{Z}$ in the very first layer of $\alpha$, channels of its output can be chosen equal to output $z_0\overline{Z}$ and $\overline{z_0}Z$. By subtracting the map $\abs{z_0}^2e_0$ from above, we may even make them equal to $z_0\overline{Z_\wedge}$ and $\overline{z_0}Z_\wedge$. By taking linear combinations of those two, we may hence make the very first layer equal
\begin{align*}
   Y_\lambda= \mathrm{re}(z_0 \overline{Z_\wedge}) +\lambda \mathrm{im}(\overline{z_0}Z_\wedge).
\end{align*}
Now, by letting each input-output-channel of the subsequent layers be a multiple of the identity, we can see to it that the output of the $L$:th layer is equal to $\psi(Y_\lambda)$, where $\psi: \C \to \C$ is any neural network applied pointwise. By the classical universality result, we can in particular make it arbitrarily close to $(Y_\lambda)^k$ for any $k \in \N$. These vectors are averaged to the so called \emph{powersum polynomials} in $Y_\lambda$, i.e.
\begin{align*}
    ps_k(Y_\lambda) = \sum_{i \geq 1} (Y_\lambda)_i^k
\end{align*}
These polynomials are, of course, exactly equal to the powersum polynomials in
\begin{align*}
    \widecheck{Y}_\lambda = \mathrm{re}(z_0 \overline{Z_\vee}) +\lambda \mathrm{im}(\overline{z_0}Z_\vee).
\end{align*}
Let us correspondingly write $\widecheck{X}_\lambda =\mathrm{re}(w_0 \overline{W_\vee}) +\lambda \mathrm{im}(\overline{w_0}W_\vee) $
Since the set of powersum polynomials generate the algebra of symmetrical polynomials \cite{eidswick1968proof}, which in turn are dense in $\calC(\calP^m)$, we conclude (due to Urysohn's separation lemma) that if $v(Z)=v(W)$ for all $v\in \calV(m)$, there must for every lambda be $\widecheck{Y}_\lambda = \widecheck{X}_\lambda$ as points in $\calP^{m-1}$, i.e. up to a permutation $\pi_\lambda$
\begin{align} \label{eq:lambda}
\widecheck{Y}_\lambda = \pi_\lambda^*\widecheck{X}_\lambda.
\end{align}
Now, simply because $S_m$ is finite, there must exist a $\pi_0$ and a sequence $\lambda_n \to 0$ with $\pi_{\lambda_n}=\pi_0$ for all $0$. Inserting $\lambda=\lambda_n$ into equation \eqref{eq:lambda} and letting $\lambda\to \infty$ we get, since $\pi_0^*$ is continuous, that
\begin{align*}
    \mathrm{re}(z_0\overline{Z_\vee}) =\pi_0^* \mathrm{re}(w_0\overline{W_\vee}).
\end{align*}
By subsequently inserting a small but non-zero $\lambda_n$ into \eqref{eq:lambda} and subtracting $\mathrm{re}(z_0\overline{Z_\vee}) =\pi_0^* \mathrm{re}(w_0\overline{W_\vee})$ from both sides, we obtain
\begin{align*}
    \lambda_n\mathrm{im}(z_0\overline{Z}_\vee) =\lambda_n\pi_0^* \mathrm{im}(w_0\overline{W_\vee}) \, \Rightarrow \, \mathrm{im}(z_0\overline{Z_\vee}) =\pi_0^* \mathrm{im}(w_0\overline{W_\vee}).
\end{align*}
Hence, $z_0\overline{Z_\vee}=w_0\overline{W_\vee}$ up to a permutation, as claimed.

\paragraph{Claim 3: $Z=W$.} Since $\abs{z_0}=\abs{w_0}$, we must have $z_0=\theta w_0$ for some $\theta \in \mathbb{S}$. By inserting this into Claim 2 and dividing by $w_0\neq 0$ (which is true due to $W \in C_\epsilon$), we get that $\theta \overline{Z_\vee}$ equals $\overline{W_\vee}$ up to a permutation. By conjugating that equality, and using that $\theta^{-1}=\overline{\theta}$, we get $Z_\vee = \theta W_\vee$ up to a permutation. This together with $z_0=\theta w_0$ however exactly means that $Z=W$ as points in $\calR\calP_0(m)$.

The claim now follows from Stone-Weierstrass.

\end{proof}

The previous lemma shows that $\calNS(m)$ is capable of approximating the function $\gamma$ in Theorem \ref{th:denseset} to arbitrary precision, as long as cases where $z_0$ is close to the origin is ignored.  In order to handle also cases in which $z_0$ is zero, we need to choose the vector unit $\psi$  in a certain manner. This is what the following, simple, lemma is for.

\begin{lemma} \label{lem:vectordense}
    Let $\epsilon>0$. There exists a function $s\in \calNC$ which vanishes for $\abs{z}<\epsilon$, equals $z$ for $\abs{z}>2\epsilon$, and satisfies $\abs{s(z)}\leq\abs{z}$ everywhere.
\end{lemma}
\begin{proof}
One easily realizes that
 \begin{align*}
    n(t)&=\tfrac{1}{\epsilon}\left(\mathrm{ReLU}(t-\epsilon) -\mathrm{ReLU}(t-2\epsilon)\right)\begin{cases}0 &\text{ if } \abs{z}<\epsilon \\
        \tfrac{t-\epsilon}{\epsilon}  &\text{ if }\epsilon \leq t<2\epsilon \\
        1 & \text{ else.}\end{cases} \\
    m(t)&=\tfrac{1}{2}\left(\mathrm{ReLU}(t-\epsilon) +\mathrm{ReLU}(t-2\epsilon)\right)\begin{cases}0 &\text{ if } \abs{z}<\epsilon \\
        \tfrac{t-\epsilon}{2}  &\text{ if }\epsilon \leq t<2\epsilon \\
        t-\tfrac{3}{2}\epsilon & \text{ else.}\end{cases}
        \end{align*}
    If follows that $m(t)+\tfrac{3}{2}\epsilon n(t)$ equals zero for $t<\epsilon$, equals $t$ for $t>2\epsilon$, and is smaller than $t$ for all $t\geq 0$. Consequently, 
    \begin{align*}
       s(z)=\left(m(z)+\tfrac{3}{2}\epsilon n(z)\right) \tfrac{z}{\abs{z}}
    \end{align*}
    fulfills the requirements of the lemma and is, due to the definition of $\rho_\C$, in $\calNC$.
\end{proof}
We can now prove the universality of our architecture.

\begin{proof}[Proof of Theorem \ref{th:universality}]
    Fix a compact, arbitrary set $K \sse \mathcal{P}^m$, $\delta>0$, and $f \in \mathcal{C}(\calRP^m)$ arbitrary.  Our goal is to show that there exists a $\Psi \in \mathcal{NR}(m)$ with $\sup_{Z \in K} \abs{\Psi(Z)-f(Z)} \leq \delta$. For future reference, set  $\omega = \sup_{Z \in K} \sup_{i \in [m]} \abs{z_i}$.
    
    By Theorem \ref{th:denseset}, there exists a function of the form
    \begin{align}
        g(Z)= \sum_{i \in [m]} \gamma(\tau_i^*Z)z_i
    \end{align}
    with $\sup_{Z \in K} \abs{f(Z)-g(Z)}<\tfrac{\delta}{2}$ and $\gamma \in \calC(\mathcal{RP}_0^m)$. Write $\omega' = \sup_{Z\in K} \sup_{i \in \tau_i^*}\abs{\gamma(\tau_i^*Z)}$, and define
    \begin{align*}
        \epsilon = \tfrac{\delta}{4(5m\omega'+2m)}
    \end{align*}
    Lemma \ref{lem:weightsdense}  proves that there exists an $\alpha \in \mathcal{NS}(m)$ with 
    \begin{align*}
        \sup_{Z \in C_\epsilon^m\cap K}\abs{\alpha(Z) - \gamma(Z)} \leq \delta' :=  \min(\tfrac{\delta}{4m\omega},1).
    \end{align*}
    Concretely, this means that 
    \begin{align}
        \abs{\alpha(\tau_i^*Z) - \gamma(\tau_i^*Z)} \leq \delta' \text{ if } \abs{z_i} \geq \epsilon. \label{eq:close}
    \end{align}
    Applying Lemma \ref{lem:vectordense}, we may further choose $\psi$ equal to $s$ as defined in that Lemma. Then, by definition, 
    \begin{align*}
        \Psi(Z) = \sum_{i \in [m]} \alpha(\tau_i^*Z) s(z_i) \in \mathcal{NR}(m).
    \end{align*}
    We now have
    \begin{align*}
        \abs{\Psi(Z)- g(Z)} \leq & \underbrace{\sum_{i: \abs{z_i}<\epsilon} \abs{\alpha(\tau_i^*Z)s(z_i)-\gamma(\tau_i^*Z)z_i }}_{(I)} + \underbrace{\sum_{i:\epsilon \leq  \abs{z_i}<2\epsilon}\abs{\alpha(\tau_i^*Z)s(z_i)-\gamma(\tau_i^*Z)z_i }}_{(II)} \\
        &+\underbrace{\sum_{i: 2\epsilon \leq \abs{z_i}} \abs{\alpha(\tau_i^*Z)s(z_i)-\gamma(\tau_i^*Z)z_i }}_{(III)}.
    \end{align*}
    Let us discuss each of these terms these terms separately.
    
    \underline{$(I)$} For this terms, we have $s(z_i)=0$, and $z_i$ is small. Therefore,
    \begin{align*}
        (I) = \sum_{i: \abs{z_i}<\epsilon} \abs{\gamma(\tau_i^*Z)z_i} \leq m \omega' \epsilon .
    \end{align*}
    
    \underline{$(III)$} On this set, $s(z_i)=z_i$. Therefore
    \begin{align*}
        (III) = \sum_{i: 2\epsilon \leq \abs{z_i}} \abs{\alpha(\tau_i^*Z)-\gamma(\tau_i^*Z)} \abs{z_i} \leq m \delta' \omega,
    \end{align*}
    due to \eqref{eq:close}.
    
    \underline{$(II)$} For these $i$, we have $\abs{s(z_i)-z_i} \leq \abs{s(z_i)}+ \abs{z_i} \leq 4\epsilon$, and $\abs{s(z_i)}\leq \abs{z_i} \leq 2\epsilon$. Again using  \eqref{eq:close}, we consequently obtain
    \begin{align*}
        (II) \leq &\sum_{i:\epsilon \leq  \abs{z_i}<2\epsilon} \abs{\alpha(\tau_i^*Z)-\gamma(\tau_i^*Z)}\abs{s(z_i)}  + \abs{\gamma(\tau_i^*Z)}\abs{s(z_i)-z_i} \leq 2m\delta' \epsilon  + 4m\omega' \epsilon \leq m(2 + 4\omega') \epsilon
    \end{align*}
    
    Using the above three estimates, and our definition of $\delta'$ and $\epsilon$, we obtain 
    \begin{align*}
        \abs{\Psi(Z)-g(Z)} &\leq \epsilon( 5m \omega' + 2m  ) +  \delta' m \omega 
    \end{align*}
    The proof is finished.
\end{proof}

\subsection{Proof of Proposition \ref{prop:NRplus}} \label{sec:Nrplus}
Here, we prove that the networks in $\calNR^+(m)$ are rotation equivariant and permutation invariant, and that the set of them includes the networks in $\calNR(m)$.

\begin{proof}[Proof of Proposition \ref{prop:NRplus}]
(i). It is clear that each $\alpha^+ \in \calNS^+(m)$ still is rotation invariant(this follows from the transition to $Z\otimes \overline{Z}$ in the very first step) and that each $\psi^+\in \calNC^+(m)$ still is rotation equivariant (this follows from the fact that $\C$-linear maps and $\rho_{\C}$ both are). Since all of the linear layers are permutation equivariant, and all nonlinearities are applied pointwise, it also obvious that they are both permutation equivariant. Because of this,
\begin{align*}
    \Psi^+(\theta \pi^*Z) &= \sum_{i \in [m]} \alpha^+(\theta \pi^*Z)_i\cdot\psi^*(\theta \pi^*Z)_i = \sum_{i \in [m]} \alpha^+(Z)_{\pi^{-1}(i)}\cdot \theta\psi^+( \pi^*Z)_{\pi^{-1}(i)} = \big\lceil k = \pi^{-1}(i) \big\rceil \\
    &= \theta\cdot \sum_{k \in [m]} \alpha^+(Z)_k\cdot \psi^+( \pi^*Z)_k = \theta\cdot\Psi^+(Z),
\end{align*}
i.e. $\Psi^+ \in \calC(\mathcal{PR}(m))$.

(ii) First, by choosing all input-output-channel pairs in the linear layers of $\psi^+$ as multiples of the identity, we can for any $\psi \in \calNC(m)$ achieve $\psi^+(Z)_i = \psi(z_i)$, $i \in [m]$. We may hence concentrate our efforts of proving that for any $\alpha \in \calNS(m)$, it is possible to choose the $S_m$-invariant layers of an $\alpha^+ \in \calNS^+(m)$ such that $\alpha^+(Z) = \alpha(\tau_i^*Z)_i$, $i\in [m]$. We do this in three steps.

\paragraph{Step 1:} We claim that there for each first linear layer $B_0$ of an $\alpha\in \calNS(m)$ exists a first linear layer $B_0^+$ of an $\alpha^+\in \calNS^+(m)$ with $$ B_0^+(T)=\sum_{i \in [m]} e_i \otimes \tau_i^*B_0(\tau_i^*T),$$ where $B_0$ is a linear layer of an $\alpha$-unit. It is enough to prove that this is true for each input-output-channel pair of the linear layer. However, this is exactly the statement of Corollary \eqref{cor:isomorphism}.

\paragraph{Step 2:} Now we claim that for each subsequent linear layer $B_0$ of an $\alpha$, there exists a corresponding linear layer $B_0^+$ of an $\alpha^+$ so that
\begin{align*}
    B_0^+(\sum_{i \in [m]}e_i \otimes v_i) =  \sum_{i \in [m]}e_i \otimes \tau_i^*B_0(\tau_i^*v_i)
\end{align*}
It is again enough to prove this for each input-output-channel pair. Each such in $B_0$ is a map $L_0 \in \calL_0(1,1)$. Hence, it suffices to show that the  the map defined by
\begin{align*}
    \calK(\sum_{i \in [m]}e_i \otimes v_i) =  \sum_{i \in [m]}e_i \otimes \tau_i^*L_0(\tau_i^*v_i)
\end{align*}
is in $\calL(2,2)$. To this end, let $\pi \in S_m$ be arbitrary. We  have
\begin{align}
     \calK(\pi^*(\sum_{i \in [m]}e_i \otimes v_i))) = \calK(\sum_{i \in [m]}e_{\pi(i)} \otimes \pi^*v_i))\label{eq:parallellLnoll} = \calK(\sum_{i \in [m]}e_i \otimes \pi^*v_{\pi^{-1}(i)}))  =  \sum_{i \in [m]}e_i \otimes \tau_i^*L_0(\tau_i^*\pi^*v_{\pi^{-1}(i)}).
\end{align}
We performed an index shift in the second step,. Now we utilize that $\tau_i \circ \pi \circ \tau_{\pi^{-1}(i)} \in \Stab(0)$ to see that
\begin{align*}
    L_0(\tau_i^*\pi^*v_{\pi^{-1}(i)}) =  L_0(\tau_i^*\pi^*\tau_{\pi^{-1}(i)}^*\tau_{\pi^{-1}(i)}^*v_{\pi^{-1}(i)}) =
   \tau_i^*\pi^*\tau_{\pi^{-1}(i)}^* L_0(\tau_{\pi^{-1}(i)}^*v_{\pi^{-1}(i)}),
\end{align*}
since $L_0$ is $\Stab(0)$-equivariant. Consequently,  \eqref{eq:parallellLnoll} is equal to
\begin{align*}
    \sum_{i \in [m]}e_i \otimes \pi^* \tau_{\pi^{-1}(i)}^*L_0(\tau_{\pi^{-1}(i)}^*v_{\pi^{-1}(i)})
   = \sum_{i \in [m]}e_{\pi(i)} \otimes \pi^* \tau_{i}^*L_0(\tau_{i}^*v_{i})  = \pi^*(\sum_{i \in [m]}e_{i} \otimes \tau_{i}^*L_0(\tau_{i}^*v_{i})) = \pi^*\calK(\sum_{i \in [m]}e_i \otimes v_i).
\end{align*}
We again performed index shifts. Thus, $\calK$ is $S_m$-equivariant, which was to be proven.

\begin{figure}
    \centering
    \includegraphics[width=.35\textwidth]{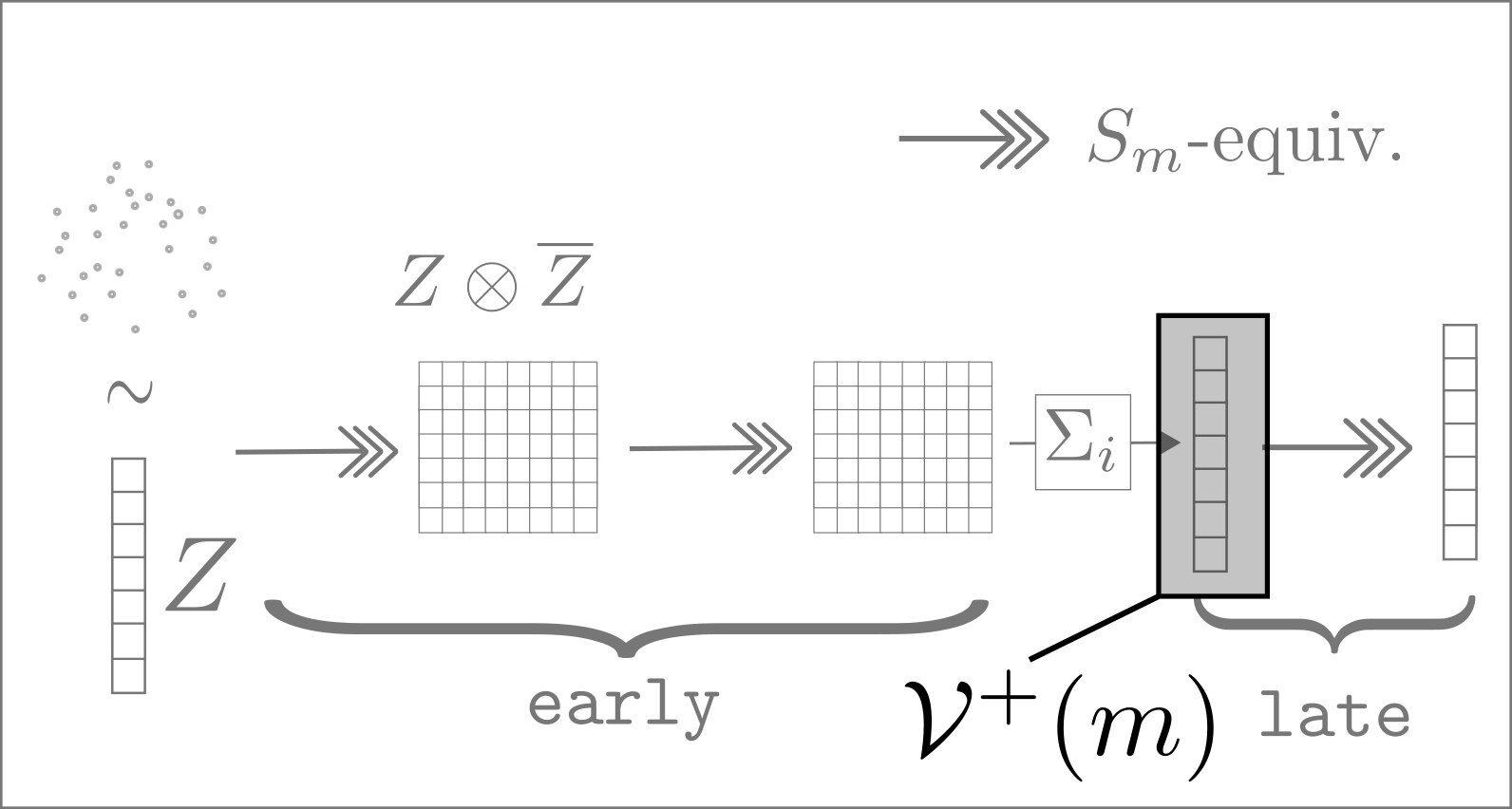}
    \caption{Definition of the space $\calV^+(m)$. \label{fig:Vmplus}}
\end{figure}
\paragraph{Step 3:}  By inductively applying Step 1 and 2, we obtain that there for every function $f$ corresponding to the early layers of a network in $\calNS(m)$, there exists a network in $\calNS(m)^+$ whose first early layers have an output
\begin{align*}
   f^+(Z)= \sum_{i \in [m]} e_i \otimes \tau_i^*(f(\tau_i^*Z)).
\end{align*}
We now carry out the summation over one of the tensor dimensions of this to obtain
\begin{align*}
    \sum_{j \in [m]} f^+(Z)_{ji} = \sum_{j \in [m]}  [\tau_i^*(f(\tau_i^*Z))]_j = \big\lceil k= \tau_i(j) \big\rceil =  \sum_{k \in [m]}  (f(\tau_i^*Z))_k
\end{align*}
Remember the definition of the space $\calV(m)$ in the proof of Lemma \ref{lem:weightsdense}. If we correspondingly define $\calV^+(m)$ as the set of functions defined by summing the output of the early layers of members of $\calNS^+(m)$-networks (see Figure \ref{fig:Vmplus}), the above shows   there for every $v \in \calV(m)$ exists a $v^+ \in \calV^+(m)$ with
\begin{align*}
    v^+(Z)_i = v(\tau_i^*Z), \quad i \in [m].
\end{align*}
 By subsequently choosing all channels in the final layers as appropriate multiples of the identity, we can therefore achieve that $\alpha^+(Z)_i = \alpha(\tau_i^*Z)$ for all $i$, which was to be proven.

\end{proof}
\subsection{The two-cloud architecture} \label{sec:twoclouds}

Here, we provide a discussion on the architectures for handling pairs of point clouds. Similarly as in the proof of the main result, we first need to equip the space of clouds of point pairs with a metric structure.

\begin{defi}
    For a subgroup of $G \sse S_m$, we let $\sim_G$ denote the equivalence relation $$(Z,X)\sim (W,Y) \Leftrightarrow \exists \, \pi \in G : (Z,X)=(\pi^*W,\pi^*Y)$$ on $\C^m \times \C^m$. We equip the set of such equivalence classes with the metric
    \begin{align*}
        d_G((Z,X),(W,Y)) = \inf_{\pi \in G} \big(\norm{Z-\pi^*W}^2 + \norm{W-\pi^*Y}^2\big)^{\sfrac{1}{2}}
    \end{align*}
    We denote the space that emerges for $G=S_m$ with $\calPP^m$, and for $G=\Stab(0)$ with $\calPP_0^m$.
    
    On $\calPP^m$ and $\calPP_0^m$ we define a further equivalence relation via $$(Z,X) \sim (W,Y) \Leftrightarrow \exists \, \theta, \omega \in \S :\ Z=\theta W, X= \omega Y. $$
    On the resulting spaces of equivalence classes, which we denote $\calRPP^m$ and $\calRPP_0^m$, we define a metric through
    \begin{align}
        d_{\S^2}((Z,X),(W,Y)) = \inf_{\theta, \omega \in \S} d((Z,X), (\theta W, \omega Y)).
    \end{align}
\end{defi}

Recall that $\calR_2(m)$ was the space of functions in $\calC(\calPP^m)$ which were rotation equivariant with respect to the first cloud, and rotation invariant to the second, and the neural network architectures $\calNR_2(m)$ and $\calNR_2^+(m)$ proposed in Section \ref{sec:NR2} of the main paper. 

The first result we wish to present for $\calNR_2(m)$ is a negative one. Its proof explicitly utilizes the basis for $\calL_2(2,1)$ provided in Section \ref{sec:spanning sets}. Hence, it might be wise to familiarize oneself with that basis before reading the proof.

\begin{prop} \label{prop:nogoNR}
    $\mathcal{NR}^2(m)$ is not dense in $\calR^2(m)$ for any $m \geq 5$.
\end{prop}

 \begin{proof}
 First, let us notice that since we only modify the architectures for calculating the weight units compared to the one-cloud case, the networks in $\calNR_2(m)$ all have the form
 \begin{align*}
     \Psi(Z,X) = \sum_{i \in [m]}\alpha(\tau_i^*Z,\tau_i^*X) \psi(z_i).
 \end{align*}
      with $\alpha$ $\Stab(0)$-invariant and  invariant to rotations of either cloud.
      
      Let us call clouds $X$ with $\sum_{i \in [m]} x_i =0$ and  $x_0=0$ \emph{centered}. Consider the basis   $(K_i)_{i \in [15]}$ of $\calL_0(2,1)$ described in Section \ref{sec:spanning sets}. All of their action on elements of the form  $X\otimes \overline{X}$ (see in particular the final paragraph of the mentioned section) are identically zero, except for $$K_1(X\otimes \overline{X})= e_0 \sum_{i \in [m]}\abs{x_i}^2, K_5(X\otimes \overline{X})= \mathbb{1} \sum_{i \in [m]}\abs{x_i}^2 \text{ and } K_{14}(X\otimes \overline{X}) =(\abs{x_i}^2)_{i \in [m]}.$$ 
       Consequently, when $X$ is centered, the very first layer of the network, and therefore the entire value $\alpha(Z,X)$, can only depend on the  norms $(\abs{x_i})_{i \in [m]}$ (and $Z$). Hence, if $X,\widetilde{X}$ are centered clouds with $\abs{x_i}=\abs{\tilde{x}_i}$ for all $i$, there must be
       \begin{align}
           \alpha(Z,X) = \alpha(Z,\widetilde{X}) \label{eq:normeq}
       \end{align}
       To increase readability, let us refer to such pairs of centered clouds as \emph{norm-equal}.

       We now show that \eqref{eq:normeq} leads to a contradiction. Consider functions of the form
        \begin{align}
            f(Z,X)= \sum_{i} \sup_{j \neq i} \inf_{k\neq j,i} a(\abs{x_j-x_k})\cdot b(\abs{z_i}) \ z_i.  \label{eq:badfunction},
        \end{align}
        where $a$ and $b$ are monotone functions. 
        That is, in words: for each $i$, we go over all of the points $x_j$, $j\neq i$, and calculate the distance to nearest neighbor which is not equal to $x_i$. We then insert those distances into $a$, choose the biggest of the resulting values, and  multiply it with $b(\abs{z_i})$ to obtain a weight for $z_i$ to use in a weighted average.
        It is not hard to realize that these are in $\calR_2(m)$.
        
        Let us be a bit more concrete and choose  $b$ to be equal to $0$ on  $[0,1/2]$ and equal to $1$ on $[1,\infty[$ and $a$ in a similar fashion be equal to $0$ on $[0,1/4]$ and equal to $1$ on $[1/2, \infty]$. Now, let $Z$ be a cloud with all points equal to $0$ except for $z_0$, which has norm $1$. We then have 
        \begin{align*}
              f(Z,X)= \sup_{j \neq 0} \inf_{k\neq j,0} a(\abs{x_j-x_k}) z_0.
        \end{align*}
        Note that since both $\rho_\C$ for all $\theta>0$ and all linear layers map $0$ to $0$, we must have $\psi(z_i)=0$ for all $i\neq 0$ and $\psi \in \calNC$. Consequently, for all $\Psi\in \calNR_2(m)$ and $Z$ as above, we have
        \begin{align}
            \Psi(Z,X) = \alpha(Z,X)\psi(z_0). \label{eq:ournetwork}
        \end{align}
         Now suppose that we can construct an norm-equal pair of  balanced clouds $X, \widetilde{X}$ with 
      \begin{enumerate}[(i)]
          \item $\abs{x_i}=\abs{\tilde{x}_i} \leq \frac{1}{2}$ for all $i$
          \item  $\sup_{j \neq 0} \inf_{k\neq j,0} a(\abs{x_j-x_k}) =1$, but $\sup_{j \neq 0} \inf_{k\neq j} a(\abs{\tilde{x}_j-\tilde{x}_k}) =0$,
      \end{enumerate}
      then $f(Z,X)=z_0$, but $f(Z,\widetilde{X})=0$. Consquently,
        \eqref{eq:normeq} would then imply that \eqref{eq:ournetwork} cannot approximate \eqref{eq:badfunction}  for both $(Z,X)$ and $(Z,\widetilde{X})$. To see that this is possible, consider a cloud $X$ with $x_0=0$, $x_{1,2}=\sfrac{1}{2}$, $x_{3,4}=\pm \sfrac{i}{2}$ and, if needed, the rest of the points arranged in a balanced fashion close to the origin. Then, $X$ is balanced, and surely fulfills (i). We would furthermore have
      \begin{align*}
          \sup_{j \neq 0} \inf_{k\neq j,0} a(\abs{x_j-x_k}) \geq  \inf_{k\neq 1,0} a(\abs{x_1-x_k}) =1, 
      \end{align*}
      since all points in the cloud not equal to $1$ are at a distance further than $\sfrac{1}{4}$ from $x_1$. Now define $\widetilde{X}$ by letting all points in $X$ be fixed, but rotating $x_3$ and $x_4$ using the same rotation $\theta$ (see Fig. \eqref{fig:centered}). Then, $(X, \tilde{X})$ surely is a norm-equal pair. However, we can rotate $x_3$ and $x_4$ in a fashion so that each point in $\tilde{X}$ has a nearest neighbor at a distance smaller than $\sfrac{1}{4}$. Consequently,
       \begin{align*}
          \sup_{j \neq 0} \inf_{k\neq j,0} a(\abs{\tilde{x}_j-\tilde{x}_k}) =0.
      \end{align*}
      This proves the proposition.

\begin{figure}
    \centering
    \includegraphics[width=.5\textwidth]{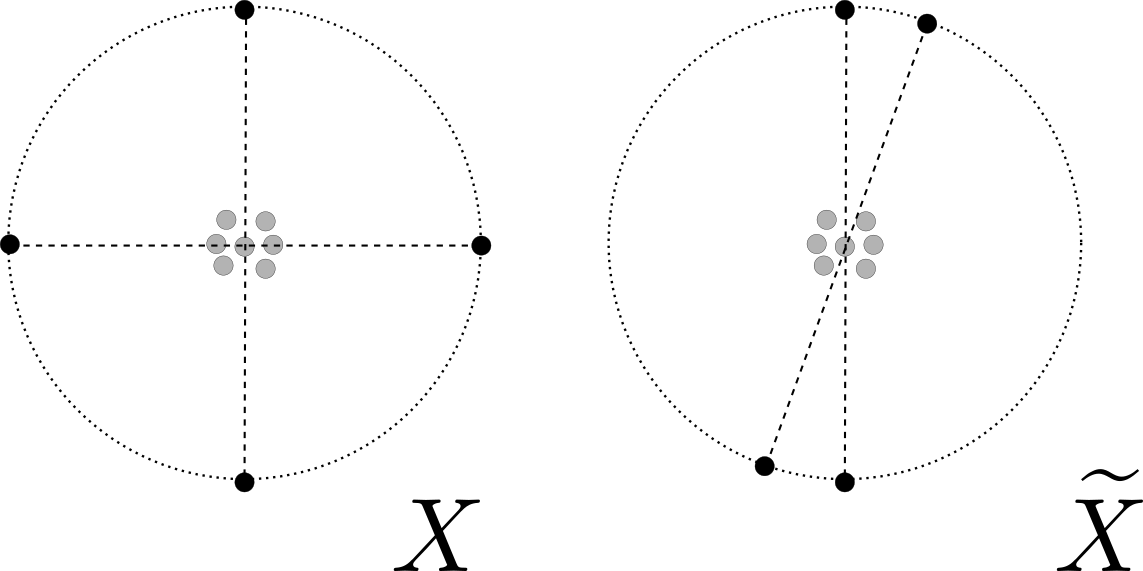}
    \caption{The norm-equal pair of centered clouds $X, \widetilde{X}$ used in the proof of Proposition \ref{prop:nogoNR}.}
    \label{fig:centered}
\end{figure}
    
    \end{proof}
    
The last proposition shows that in order to prove a universality result, we need to restrict the set of functions we want to approximate. The following theorem describes one such possible restriction: If we are only concerned with pairs $(Z,X)$ for which $\abs{z_i} \lesssim \abs{x_i}$, i.e. cloud pairs for which points close to the origin in $X$ correspond to points close to the origin in $Z$, we again obtain universality
\begin{theo} 
     For $a>0$, define the set
    \begin{align*}
        D_{a} = \set{ (Z,X) \in \mathcal{PP}^m \, \vert \, a\abs{z}_i \leq \abs{x}_i \ , i \in [m]}.
    \end{align*}
    Then, both $\mathcal{NR}_2(m)$ and $\mathcal{NR}_2^{+}(m)$ are dense in the space of $\calC(D_a)$-functions which are rotation-equivariant with respect to the first cloud.
\end{theo}

\begin{proof}
 The proof follows the beats of Theorem \ref{th:universality}  very closely. We will therefore only provide a sketch, concentrating on the parts of the argument which are significantly different.

One proves   $\mathcal{NR}_2(m)\sse \mathcal{NR}_2^{+}(m)$ just as the corresponding result for single cloud networks. Hence, it is enough to prove universality for $\mathcal{NR}_2(m)$. To do that, on first generalizes Theorem \ref{th:denseset} by proving that the set of functions
\begin{align*}
    g(Z,X) = \sum_{i \in [m]}\gamma(\tau_i^*Z,\tau_i^*X) z_i,
\end{align*}
where $\gamma$ is arbitrary in the space of  $\calC(\mathcal{RPP}^m_0)$, is dense in $\calR_2(m)$ The proof is more or less verbatim equal to the proof of the $\calR(m)$-result : One first proves that we can approximate the function using a polynomial in $\calR^2(m)$, similarly as in Lemma \ref{lem:equipolys}. The proof then boils down to rewriting polynomials of the form
\begin{align*}
    \sum_{\pi \in S_m} Z^{\pi^*\alpha_0}\overline{Z}^{\pi^*\beta_0} Z^{\pi^*\alpha_1}\overline{X}^{\pi^*\beta_1}
\end{align*}
with $\abs{\alpha_0}=\abs{\beta_0}+1$ and $\abs{\alpha_1} = \abs{\beta_1}$. It should be stressed that the last equalities are consequences of the 'separate equivariance' property.

Next, one moves on to generalizing Lemma \ref{lem:weightsdense}. One proves that the space $\mathcal{NS}^2(m)$ of two-cloud $\alpha$-units is dense in $\mathcal{C}(C_{a,\epsilon})$, where
\begin{align*}
    C_{a,\epsilon} = \set{ (Z,X) \in D_{a} \, \vert \, \abs{z_0} \geq \epsilon}.
\end{align*}
Note that if $(Z,X) \in C_{a,\epsilon}$, we also have $\abs{x_0}\geq a \abs{z_0} > 0$.

The idea of the proof is again to apply the Stone-Weierstrass theorem, with the functions $\calV_2(m)$ that are given by outputs of $\alpha$-units after the invarization step as the function set $S$ (see the proofs of Lemma \ref{lem:weightsdense} and Proposition \ref{prop:NRplus}, as well as Figures \ref{fig:Vm} and \ref{fig:Vmplus}).
To do this, let us first note that by letting the very first layer of $\alpha$ only depend on either cloud, and applying the same steps as before, we get that if $v(Z,X)=v(W,Y)$ for all $v \in \calV_2(m)$, we must have $\abs{z_0}=\abs{w_0}$ and $\abs{x_0}=\abs{y_0}$. Now
notice that for every $\lambda>0$, we can also choose the output of the very first linear layer of $\alpha$ equal to
\begin{align*}
    z_0 \overline{Z_\wedge} + \lambda x_0 \overline{X_\wedge}, \quad \overline{z_0} Z_\wedge + \lambda \overline{x_0} X_\wedge,
\end{align*}
using the same notation as in the previous proof.
By subsequently following the same arguments as in the one-cloud proof, we see that there must be
\begin{align}
     z_0 \overline{Z_\vee} + \lambda x_0 \overline{X_\vee} = \pi_\lambda^*(w_0 \overline{W_\vee} + \lambda y_0 \overline{Y_\vee}) \label{eq:lambdaperm}
\end{align}
for some permutation $\pi_\lambda$, possibly dependent on $\lambda$. By applying the same trick as we did to the real and imaginary parts of $z_0\overline{Z_\vee}$ and $w_0\overline{W_\vee}$ to conclude that they were equal to each other up to a permutation,  we conclude that there exists a \emph{common} $\pi_0 \in S_m$ with
\begin{align*}
    z_0 \overline{Z_\vee} = \pi_0^* w_0 \overline{W_\vee}, \quad x_0 \overline{X_\vee} = \pi_0^*y_0\overline{Y_\vee}.
\end{align*}
We may now proceed as before -- notice that we can divide by both $z_0$ and $x_0$, since they are both unequal to $0$.

Now, the final argumentation proceeds just as in the proof of Theorem \ref{th:universality}.
\end{proof}

\section{Experiments} \label{sec:exp}

We implemented ZZ-net in PyTorch \cite{pytorch_NEURIPS2019_9015} using PyTorch Lightning \cite{pytorch_lightning_falcon_2019}. For the essential matrix problem we performed hyper parameter tuning using Ray Tune \cite{raytune_liaw2018tune}.

\subsection{Estimating rotations between noisy point clouds} 
\label{exp:toy}
Here, we provide some additional information on experiments on the toy problem.

\paragraph{Data generation}\label{sec:data_gen}
A cloud $Z$ is formed of $m=100$ points distributed on a random triangle. These are subsequently rotated to a cloud $X$ by a random rotation $\theta \in \S$, and low-level inlier noise is added to both clouds. We subsequently, with a probability $r$, exchange each correspondence with an outlier $(\hat{z}_i, \hat{x}_i)$ chosen completely at random. An example of a resulting pair for $r=0.4$ is shown in Figure~\ref{fig:cloudpair}. We generate $2000$, $500$ and $300$ cloud pairs for training, validation and testing, respectively. 
Step by step, the generation procedure is as follows:
\begin{itemize}
    \item To generate the original cloud, without outliers, we first choose three points uniformly randomly on the unit disk - these are the corners of the triangle.
    \item Next, we choose $m=100$ new points uniformly randomly on the unit disk. For each of the points, we choose one of the three sides of the triangle, and orthogonally project the point onto that side. This leaves us with an inlier cloud $Z_{\mathrm{in}}$.
    \item Next, a rotation $\theta \in \S$ is chosen uniformly at random, and we define the other cloud as $X_{\mathrm{in}}=\theta Z_{\mathrm{in}}$. We add independent Gaussian noise to each of the points in either cloud, with a standard deviation of $\sigma=0.03$.
    \item Then, we go through the point pairs, throwing each one out with a probability $r$. The ones that are thrown out are replaced with a pair of points  $(z_i, x_i)$ independently chosen uniformly on the unit disk.
\end{itemize}

\paragraph{Comparison models}\label{sec:comp_mod}
Here we outline the two comparative methods for the experiments on rotation estimation.
The first one is a PointNet with $5$ equivariant layers and a head with $5$ fully connected layers, with additional learnable batch normalization layers. The model as a whole has around $34K$ parameters. We also consider a model better adapted to handle outliers, incorporating an attentive context normalization \cite{sun-cvpr-2020} with $7$ layers, for a total of around $11K$ parameters. We refer to the latter as `ACNe\textminus', since it lacks a lot of mechanisms (such as group normalization, skip connections, and other things) compared to the actual ACNe model. To reiterate, we think it would be dishonest to claim that we in this experiment compare our method with\cite{sun-cvpr-2020}. Our aim is rather to show that our approach can compete also with networks tailor-made for outlier-heavy scenarios. Both of these models take in the correspondences as vectors in $\R^4$, used as the channels in the first layer, and outputs two real scalars, which we reinterpret as a complex outputs. They are in particular not rotation equivariant.  

\paragraph{The 'ACNe\textminus'-model} Let us discuss our implementation of an 'ACNe-architecture' \emph{inspired} by \cite{sun-cvpr-2020}. The ACNe\textminus model consists of so called ACNe-units. In each such, each point in the input is first fed through one linear layer and an activation function to produce a cloud of features $F \in (\R^C)^m$. These weights are then fed through two different linear layers to produce two vectors $v_1,v_2 \in \R^m$. A sigmoid is applied pointwise to $v_1$ to produce the \emph{local weight vector} $w_1$. $\mathrm{SoftMax}$ is applied to $v_2$ to produce a \emph{global} weight vector $w_2$. These are then multiplied pointwise, and normalize to sum to one, to produce the final weight vector $w$. 

This vector is subsequently used to \emph{context normalize} the feature cloud $F$. That is, each channel is normalized to have zero mean and unit variance, with respect to the probability distribution defined by $w$. That is, with $\hat{F}=\sum_{j \in [m]}w_j F_j$, the $k$:th channel of the output of the ACNe unit is equal to 
\begin{align*}
    G_i^k  = \frac{ F_i^k -\hat{F}^k}{\bigg( \sum_{i\in [m]} w_i(F_i^k -\hat{F}^k)^2\bigg)^{\sfrac{1}{2}}}.
\end{align*}

The entire 'ACNe\textminus'-net has two additional steps: First, the initial input is fed through one perceptron layer before being fed to the first ACNe-unit. The actual output of the net is formed by the weighted average $\hat{F}$ of the final ACNe unit. This is different from \cite{sun-cvpr-2020}, where the output of the final layer is processed further in a problem-dependent manner.

\paragraph{Model sizes} For the broad model, the number of channel in the early layers are both equal to $4$, the late layers have $4$, $16$, $4$ and $1$ channels, respectively.  The vector unit layers have $32$ and $1$ channel, respectively. 

For the deep model, each $\calR^2(m)$-unit has $4$ channels in the early layer. The late layers in the two earlier units have $4$, $8$ and $4$ units each -- the final unit instead as late layers with $4$, $8$ and $1$ channels, respectively. The first two vector layers have $4$ channels, whereas the last has $1$. 

The permutation equivariant layers of the PointNet have $32$, $64$, $128$, $64$, $64$ and $64$ layers.  The layers of the fully connected head have $64$, $32$, $16$ and $2$ channels. We use max-pooling in between the permutation-equivariant layers and the fully connected head.

The layers of the 'ACNe\textminus' model have $4$, $32$, $32$, $64$, $64$, $32$, $32$ and $2$ layers, respectively.

\paragraph{Nonlinearities} We use the $\mathrm{ReLU}$ as a non-linearity for the PointNet, and leaky $\mathrm{ReLU}$s (where the slope parameter is set to the PyTorch standard of $.01$) for our models and the perceptrons in the 'ACNe\textminus'-model.

In addition to the mechanisms described in the main paper, we choose, for the deep and broad model, to normalize each channel of the weight unit, which is a vector in $\C^m$, to have $\ell_2$-norm $1$. We found this useful to prohibit the model to not get stuck at outputs of very small magnitudes. 
The learnable $\theta$-parameters in the complex $\mathrm{ReLU}$s are initalized to  $0.1$.

\paragraph{Training details} For the training of the PointNet, we use a stochastic gradient descent with a momentum of $0.9$. The learning rate is set to $10^{-3}$ and we train it for $400$ epochs.

For the training of the ACNe model, we use Adam \cite{adam}. The learning is initially set to $10^{-3}$, and halved after $200$ and $300$ epochs. We train it for $400$ epochs.

The broad and deep models are trained using Adam. We set the initial learning rate to $5\cdot 10^{-3}$, and half it after $70$ and $150$ epochs. We train it for $300$ epochs. 

All models are evaluated at the final epoch, with the exception of the experiment of the broad model for $r=0.8$, which severly overfitted the data (the final model had scores $0$, $0$ and $.02$ on the three metrics). Therefore, we (manually) stopped it early after 120 epochs, when the validation loss still was low.

\subsection{Essential Matrix Estimation} \label{sec:sup_exp_ess}
In this section we present more information on the experiment on essential matrix estimation from
Section~\ref{sec:exp_ess}

\begin{table}
\begin{minipage}{.48\linewidth}
    \centering
\begin{tabular}{|r || c | c | c | c |}
\hline
     Max. test rot. $a=$& $\ang{0}$ & $\ang{30}$ & $\ang{60}$  & $\ang{180}$ \\
   \hline ZZ-net (Ours)  & 0.15 & 0.15 & \textbf{0.16} & \textbf{0.15}\\
   ACNe         & \textbf{0.58} & \textbf{0.16} & 0.087 & 0.0096 \\
   CNe          & 0.30 & 0.077 & 0.058& 0.0 \\
   OANet        & 0.30 & 0.14 & 0.038 & 0.0 
   \\ \hline
\end{tabular}
\caption{Essential matrix estimation. mAP at $w=\ang{10}$ error in the estimated translation and rotation vectors for different values of
image plane rotations $a$ at test time.
\label{tab:essential10}}
\end{minipage}
\hfill
\begin{minipage}{.48\linewidth}
    \centering
\begin{tabular}{|r || c | c | c | c |}
\hline
     Max. test rot. $a=$& $\ang{0}$ & $\ang{30}$ & $\ang{60}$  & $\ang{180}$ \\
   \hline ZZ-net (Ours)  & 0.33 & \textbf{0.33} & \textbf{0.33} & \textbf{0.33} \\
   ACNe         & \textbf{0.72} & 0.32 & 0.20 & 0.054 \\
   CNe          & 0.50 & 0.21 & 0.15 & 0.022 \\
   OANet        & 0.50 & 0.30 & 0.12 & 0.026
   \\ \hline
\end{tabular}

\caption{Essential matrix estimation. mAP at $w=\ang{30}$ error in the estimated translation and rotation vectors for different values of
image plane rotations $a$ at test time.
\label{tab:essential30}}
\end{minipage}
\end{table}

\paragraph{Loss function}
Let $\{(\xi_1, \xi_2)\}$ 
denote a set of virtual matches (generated as the authors of OANet do by using the OpenCV \texttt{correctMatches} function), where $\xi_1$ and $\xi_2$ are in $\R^2$ and 
$\tilde\xi_1$ and $\tilde\xi_2$ are the homogeneous representations.
Then the symmetric squared epipolar loss of an estimated essential matrix $E$ is
\[
    \frac{(\tilde\xi_2^T E\tilde\xi_1)^2}{(E\tilde\xi_1)^2_{[0]} + (E\tilde\xi_1)^2_{[1]}} +
    \frac{(\tilde\xi_2^T E\tilde\xi_1)^2}{(E^T\tilde\xi_2)^2_{[0]} + (E^T\tilde\xi_2)^2_{[1]}},
\]
which we average over the set of virtual matches.

\paragraph{Evaluation metric}
The mAP score proposed by~\cite{moo-cvpr-2018} is obtained by first, for equispaced angle values
$v=\ang{5},\ang{10},\ldots,\ang{30}$, calculating 
the proportion of estimated $E$-matrices that 
have an error in angle of both the translation vector and the rotation
axis vector below $v$. The obtained proportion can be called the precision at $v$. 
The mAP at an angle $w$ is then obtained by averaging the precision at all $v\leq w$.

\paragraph{Further results} \label{sec:ess_further}
We present mAP scores at $\ang{10}$ and $\ang{30}$ in Tables~\ref{tab:essential10} and \ref{tab:essential30}.
Once again our results are averaged over two runs. The maximum difference between the scores in these two runs
for mAP at $\ang{10}$ was 0.03 and at $\ang{30}$ it was 0.02.

\paragraph{Model details}
The layer structures are as follows. The backbone $\mathcal{B}$ has three ZZ-units.
The first has two early layers which 
both have 8 output channels, two late layers which have 8 and 3 output channels, and two vector layers which
have 8 and 3 output channels. The second ZZ-unit has two early layers again both with 8 output channels,
two late layers with 8 and 3 output channels, and two vector layers with 8 and 3 output channels.
The last ZZ-unit has one early layer with 8 output channels, one late layer with 8 output channels and one
vector layer with 8 output channels.
We add skip connections so that the input to each ZZ-unit is both the input to the previous unit as well as the previous unit's output.

The equivariant angle predictor $\mathcal{E}$ consist of one ZZ-unit. It has one early layer with 8 output channels,
one late layer with 1 output channel and two vector layers with 8 and 1 output output channels.
The output of $\mathcal{E}$ is averaged over the point cloud to predict one complex number, interpreted as one angle.

The invariant angle predictor $\mathcal{I}$ takes the outputted $\alpha^+$-weights of the backbone
(which are rotation invariant) as input and applies a PointNet/Deepset to it. Here the real and imaginary
channels are 
treated like any other channel, i.e. the number of input channels to $\mathcal{I}$ is twice the number of
(complex) output channels of $\mathcal{B}$. $\mathcal{I}$ consists of three layers, with 32, 64 and 4 output channels respectively. The output of $\mathcal{I}$ is averaged over the point cloud to get permutation invariance and the 4 outputted real numbers are then reinterpreted as 2 complex numbers or angles.

We add context normalization (CN) \cite{moo-cvpr-2018}
between the early and late layers as well as after the vector layers
in each ZZ-unit. CN normalizes the features within a point cloud to mean 0 and variance 1.

\paragraph{Training details} \label{sec:ess_train_setup}
We implemented our model in Pytorch using Pytorch Lightning.
We used Ray Tune to find reasonable hyperparameters and then retrained the method with those.

We train the model for 30 epochs using early stopping on the validation loss.
We use a learning rate of 0.01 and train using Adam. We use a batch size of 1 due to the heavy memory need.

For all comparisons we use the settings supplied by the respective authors, except for the number of
training iterations which we change to 100000 to compare with our method (30 epochs corresponds to 
$30\cdot 3302 = 99060$ iterations).

\section{Spanning sets for spaces of \texorpdfstring{$\Stab(0)$}{Stab(0)}-equivariant linear maps} \label{sec:spanning sets}

Here we present explicit spanning sets for the spaces $\calL_0(k,\ell)$ from Section \ref{sec:spanning sets}. They are obtained via applying the isomorphism given in \ref{cor:isomorphism} to the spanning sets of $\calL(k,\ell+1)$ described in \cite{maron2018invariant}.

\paragraph{$\calL_0(0,0)$} This is simply the space scalars, i.e. $\K$.

\paragraph{$\calL_0(1,0)$} The space has dimension $B_2\leq2$. A basis is given by
\begin{align*}
    \mu_0(v)=v_0, \ \mu_1(v) =\sprod{\one,v}.
\end{align*}

\paragraph{$\calL_0(0,1)$} The space has dimension $B_2\leq 2$. A basis is given by
\begin{align*}
    w_0=e_0, \ w_1 =\one.
\end{align*}

\paragraph{$\calL_0(2,0)$} The space has dimension $B_3\leq 5$. A basis is given by
\begin{align*}
     \lambda_0(T) &= \sprod{\one,T\one}, \ \lambda_1(T) = \sprod{\one,\diag(T)}, \ \lambda_2(T) = T_{00} \\ \lambda_3(T)&=\sprod{e_0,T\one}, \ \lambda_4(T)=\sprod{e_0,T^T\one}.
\end{align*}

\paragraph{$\calL_0(1,1)$} The space has dimension $B_3 \leq 5$. A basis is given by
\begin{align*}
    L_0(v) &= \sprod{\one,v}\one, \ L_1(v) = v, \ L_2(v) = v_0 e_0 \\ L_3(v)&=\sprod{\one, v}e_0, \ L_4(T)=v_0 \one.
\end{align*}

\paragraph{$\calL_0(0,2)$} This space has dimension $B_3\leq 5$. A basis is given by
\begin{align*}
    T_0 &= \one \otimes \one, \ T_1 =  \diag^*(\one), T_2 = e_0 \otimes e_0 \\
    T_3 &= e_0 \otimes \one , T_3 = \one \otimes e_0
\end{align*}
where $\diag^*: \K^m \to \K^m \otimes \K^m$ is the dual operator of $\diag$. Concretely, $\diag^*(v)$ is the tensor with diagonal $v$. 

\paragraph{$\calL_0(2,1)$} The space has dimension $B_4\leq 15$. If we let $\lambda_{i}$ denote the basis of $\calL_0(2,0)$ from above, the first 10 basis elements are given by
\begin{align*}
   K_i(T) = \lambda_i(T) e_0 , \ K_{4+i}(T) = \lambda_i(T) \one , \ i =0, \dots,4.
\end{align*}
The final five are given by
\begin{align*}
    K_{10}(T) &= Te_0, \ K_{11}(T) = T^Te_0, \ K_{12}(T) = T\one \\
    K_{13}(T) &= T^T\one, \ K_{14}(T) = \diag(T)
\end{align*}

\paragraph{$\calL_0(1,2)$} The space has dimension $B_4\leq 15$. If we let $T_i$ denote the basis of $\calL_0(0,2)$ from above, the first 10 basis elements are given by
\begin{align*}
   L_i(v) = v_0 T_i, \ L_{4+i}(v) = \sprod{\one,v} T_i , \ i =0, \dots,4.
\end{align*}
The final five are given by
\begin{align*}
    L_{10}(v) &= e_0 \otimes v, \ L_{11}(T) = v \otimes e_0 \ L_{12}(T) = \one \otimes v \\
    L_{13}(T) &= v \otimes \one, \ L_{14}(T) = \diag^*(v)
\end{align*}

\paragraph{$\calL_0(2,2)$} The space has dimension $B_5\leq 52$. If we let $T_i$ denote the basis of $\calL_0(0,2)$  and $\lambda_i$ the one of $\calL_0(2,0)$, from above, the first 25 basis elements are given by
\begin{align*}
    \calK_{5i+j}(T)= \lambda_j(T)T_i , \  i,j =0, \dots,4.
\end{align*}
Letting $K_i$ denote the basis of $\calL_0(2,1)$  and $L_i$ the one of $\calL_0(1,2)$, the next 25 are given by
\begin{align*}
    \calK_{25+5i+j}(T) = L_{10+i}(K_{10+j}(T)), i,j=0, \dots, 4
\end{align*}
The final two are given by
\begin{align*}
    \calK_{50}(T)=T, \calK_{51}=T^T.
\end{align*}

\paragraph{Applying \texorpdfstring{$\calL(2,1)$}{L(2,1)}-maps to \texorpdfstring{$Z\otimes \overline{Z}$}{ZZ}.} When describing the $\calNS(m)$-architecture, we argued that the very first layer of an $\calNS(m)$-unit can be applied without calculating $Z\otimes \overline{Z}$. Let us show this. We have
\begin{align*}
    \lambda_0(Z\otimes \overline{Z}) &= \big\vert \sum_{i\in [m]} z_i \big \vert^2,   \quad \lambda_1(Z\otimes \overline{Z}) = \sum_{i\in [m]} \abs{z_i}^2, \quad  \lambda_2(Z\otimes \overline{Z}) = \abs{z_0}^2 \\
    \lambda_3(Z\otimes \overline{Z}) &=z_0 \cdot \overline{\sum_{i \in [m]}z_i}, \quad\lambda_4(Z\otimes \overline{Z}) =\overline{z_0} \cdot \sum_{i \in [m]}z_i.
\end{align*}
Clearly, all of these expressions can be calculated directly from $Z \in \C^m$, which implies that the same is true for $K_i$, $i=0, \dots, 9$. As for the last five maps, we have
\begin{align*}
    K_{10}(Z\otimes \overline{Z}) &= \overline{z_0}Z, \quad K_{11}(Z\otimes \overline{Z}) = z_0 \overline{Z}, \quad K_{12}(Z \otimes \overline{Z}) = \overline{\sum_{i \in [m]}z_i} \cdot Z\\
    K_{13}(Z \otimes \overline{Z}) &= \big(\sum_{i \in [m]}z_i \big) \cdot \overline{Z}, \quad K_{14}(Z) = (\abs{z_i}^2)_{i \in [m]}
\end{align*}
These expressions can clearly also be calculated without actually accessing $Z\otimes \overline{Z}$ as a tensor.

\end{document}